\documentclass{article}

\usepackage[preprint]{neurips_2026}

\usepackage[utf8]{inputenc} 
\usepackage[T1]{fontenc}    
\usepackage[backref=page]{hyperref}       
\usepackage{url}            
\usepackage{booktabs}       
\usepackage{amsfonts}       
\usepackage{nicefrac}       
\usepackage{microtype}      
\usepackage{xcolor}         
\usepackage{enumitem}

\usepackage{amsmath, amsthm, amsfonts, bm}
\usepackage{mathtools}
\usepackage{thmtools}
\usepackage{thm-restate}
\usepackage{xcolor}
\usepackage{xspace}
\usepackage[framemethod=TikZ]{mdframed}

\def\eqref#1{equation~\ref{#1}}

\def\1{\bm{1}}

\DeclareMathAlphabet{\mathsfit}{\encodingdefault}{\sfdefault}{m}{sl}
\SetMathAlphabet{\mathsfit}{bold}{\encodingdefault}{\sfdefault}{bx}{n}

\def\gD{{\mathcal{D}}}

\def\gL{{\mathcal{L}}}

\def\gP{{\mathcal{P}}}

\def\gU{{\mathcal{U}}}

\newcommand{\E}{\mathbb{E}}

\newcommand{\R}{\mathbb{R}}

\DeclareMathOperator*{\argmax}{arg\,max}

\newcommand{\xhdr}[1]{{\noindent\bfseries #1}.}
\newcommand{\cut}[1]{}

\newcommand{\namelong}{\textsc{Flow Map $Q$-Guidance}\xspace}
\newcommand{\nameshort}{\textsc{FMQ}\xspace}

\newcommand{\namelonginf}{\textsc{$Q$-Guided Beam Search}\xspace}
\newcommand{\nameshortinf}{\textsc{QGBS}\xspace}

\newcommand{\nameshorttr}{\textsc{FMQ}\xspace}

\newcommand{\ddd}{\mathrm{d}}

\newcommand\joey[1]{\noindent{\color{orange} {\bf \fbox{Joey}} {\it#1}}}

\mdfdefinestyle{MyFrame}{%
    linecolor=black,
    outerlinewidth=.3pt,
    roundcorner=5pt,
    innertopmargin=1pt,
    innerbottommargin=1pt,
    innerrightmargin=1pt,
    innerleftmargin=1pt,
    backgroundcolor=black!0!white}

\mdfdefinestyle{MyFrame2}{%
    linecolor=white,
    outerlinewidth=1pt,
    roundcorner=1pt,
    innertopmargin=3pt,
    innerbottommargin=2pt,
    innerrightmargin=7pt,
    innerleftmargin=7pt,
    backgroundcolor=black!3!white}

\mdfdefinestyle{MyFrameEq}{%
    linecolor=white,
    outerlinewidth=0pt,
    roundcorner=0pt,
    innertopmargin=0pt,
    innerbottommargin=0pt,
    innerrightmargin=7pt,
    innerleftmargin=7pt,
    backgroundcolor=black!3!white}

\usepackage{cleveref}

\usepackage[normalem]{ulem}

\usepackage{algorithm}
\usepackage{algorithmic}
\usepackage{amsmath}
\usepackage{amssymb}
\usepackage{bbm}
\usepackage{thm-restate}

\newtheorem{theorem}{Theorem}[section]

\newtheorem{definition}[theorem]{Definition}

\usepackage{booktabs}
\usepackage{multirow}
\usepackage{wrapfig}

\hypersetup{
	colorlinks=true,       
	linkcolor=blue,        
	citecolor=blue,        
	filecolor=magenta,     
	urlcolor=blue         
}

\renewcommand*{\backrefalt}[4]{%
    \ifcase #1 \footnotesize{(Not cited.)}%
    \or        \footnotesize{(Cited on page~#2)}%
    \else      \footnotesize{(Cited on pages~#2)}%
    \fi
    }

\title{Aligning Flow Map Policies with Optimal $Q$-Guidance}

\author{%
  Christos Ziakas\textsuperscript{1}\thanks{Correspondence to c.ziakas24@imperial.ac.uk } \quad 
  Alessandra Russo\textsuperscript{1} \quad 
  Avishek Joey Bose\textsuperscript{1, 2} \\
  \\
  \textsuperscript{1}Imperial College London \quad 
  \textsuperscript{2}Mila 
}
\begin{document}

\maketitle

\begin{abstract}

\looseness=-1
Generative policies based on expressive model classes, such as diffusion and flow matching, are well-suited to complex control problems with highly multimodal action distributions. Their expressivity, however, comes at a significant inference cost: generating each action typically requires simulating many steps of the generative process, compounding latency across sequential decision-making rollouts. We introduce \emph{flow map policies}, a novel class of generative policies designed for fast action generation by learning to take arbitrary-size jumps---including one-step jumps---across the generative dynamics of existing flow-based policies. We instantiate flow map policies for offline-to-online reinforcement learning (RL) and formulate online adaptation as a trust-region optimization problem that improves the critic's $Q$-value while remaining close to the offline policy. We theoretically derive \namelong{} (\nameshort), a principled closed-form learning target that is optimal for adapting offline flow map policies under a critic-guided trust-region constraint. We further introduce \namelonginf{} (\nameshortinf), a stochastic flow-map sampler that combines renoising with beam search to enable iterative inference-time refinement. Across $12$ challenging robotic manipulation and locomotion tasks from OGBench and RoboMimic, \nameshort{} achieves state-of-the-art performance in offline-to-online RL, outperforming the previous one-step policy MVP by a relative improvement of $21.3\%$ on the average success rate.


\cut{
Flow-matching policies have emerged as a powerful paradigm for offline-to-online reinforcement learning (RL), yet incorporating Q-value guidance into their generation remains purely heuristic. In this work, we formalize one-step flow map policies and derive a closed-form step that maximizes the Q-value under a trust-region constraint. Building on these theoretical contributions, we propose \namelong (\nameshort), a method for aligning flow map policies with optimal  Q-guidance. In particular, we efficiently distill the optimally aligned offline policy into the online flow map policy, while anchoring the trust region to the offline behavioral prior. In addition, we introduce a beam search combining stochastic renoising with the closed-form step for iterative refinement. FMQ outperforms the state-of-the-art offline-to-online RL methods in robotic manipulation and locomotion tasks across OGBench and RoboMimic benchmarks.}

\end{abstract}

\vspace{0.5em}

\vspace{0.3em}

\section{Introduction}
\label{sec:introduction}

\looseness=-1
The supreme promise of offline reinforcement learning (RL) is that effective policies can be bootstrapped in a scalable data-driven manner without costly environment interaction~\citep{levine2020offline}. This scaling philosophy is central to modern \emph{data-driven} reinforcement learning~\citep{kumar2019data,fu2020d4rl} that utilizes ever-growing diverse offline datasets~\citep{collaboration2023open}, and now powers learning policies in high-impact applications from dialogue~\citep{jaques2019way} to robotic navigation~\citep{kahn2018composable}. Indeed, imitating the highly multi-modal action distribution of expert behavior policies in such complex control problems necessitates the use of expressive policy classes that go beyond restrictive unimodal Gaussian actors~\citep{zhu2023diffusion,wang2022diffusion}. 

\looseness=-1
Generative policies based on dynamic mass transport, such as diffusion models~\citep{sohl2015deep,song2020score} and flow-matching~\citep{liu2022flow,lipman2022flow,stochasticinterpolants}, provide a compelling alternative to Gaussian policies as they learn to map a simple base distribution into a rich state-conditioned action distribution~\citep{chi2025diffusion}. The expressivity gains of generative policies make them particularly favorable for offline and offline-to-online RL~\citep{fujimoto2021minimal,tarasov2023revisiting}, where the policy must first model diverse behaviors from a static dataset and then improve through interaction. However, the price of expressive generative policies is computationally expensive inference-time simulation. More precisely, generating actions requires numerically integrating dynamics from noise to action and is executed at every environment step~\citep{yang2023policy}---inhibiting deployment in online and real-world settings~\citep{zhan2024transformation,zhan2025physics}.

\begin{figure}[t]
    \centering
    \includegraphics[width=0.48\textwidth]{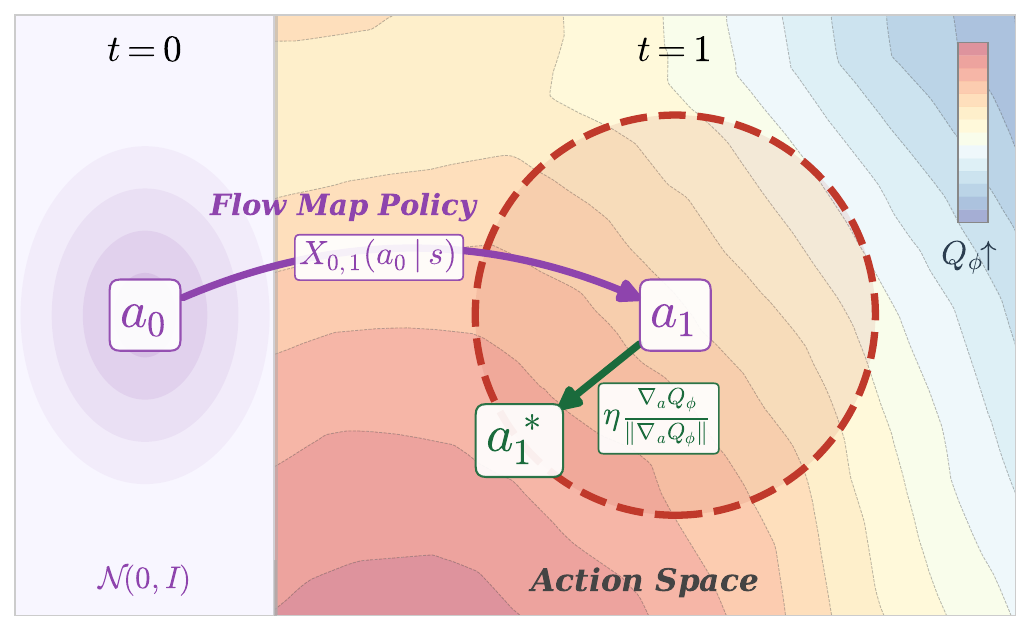}
    \hfill 
    \includegraphics[width=0.48\textwidth]{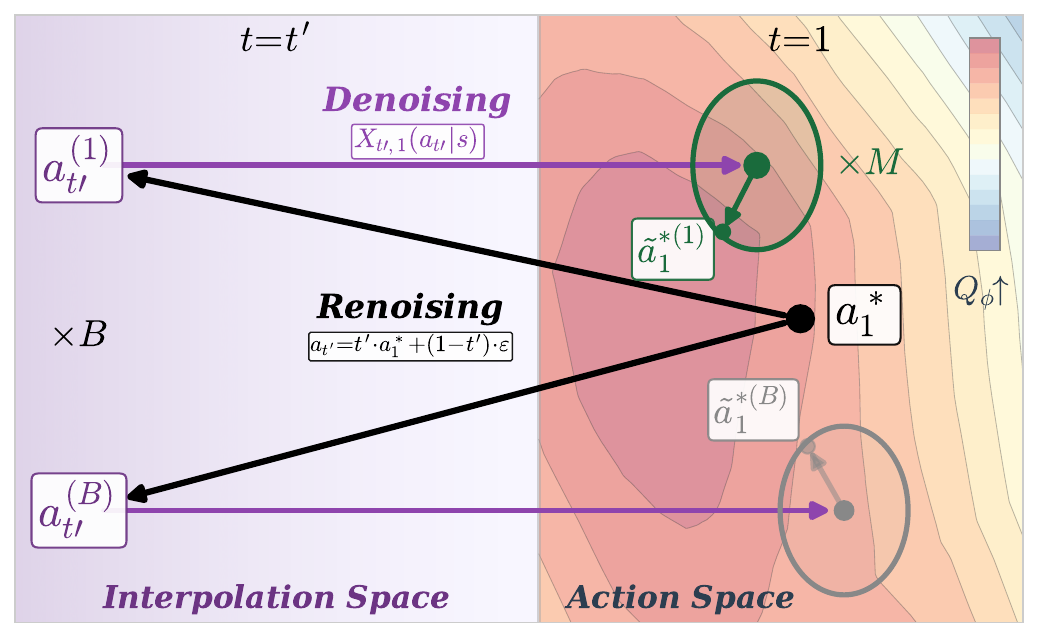}
    \vspace{-5pt}
    \caption{ \small \looseness=-1 (Left) \textbf{\nameshorttr}: one-step flow map policy transports noise $a_0$ to action $a_1$; then, trust-region projection displaces action $a_1$ to $a_1^*$ that maximizes $Q$-value. (Right): \textbf{\nameshortinf} ($M{=}1$, $B{=}2$): renoising corrupts $a_1^*$ into $B$ intermediate states $a_{t'}$, which the flow map policy then denoises to generate $B$ candidate actions; candidates are updated via the optimal trust-region displacement to maximize $Q_\phi$, and the highest-valued $M$ actions are selected.}
    \label{fig:simple_side_by_side}
    \vspace{-5pt}
\end{figure}

\looseness=-1
Addressing the inference latency of generative policies is of critical interest in order to extract maximal utility from generative policies, with several recent efforts attempting to learn one-step offline policies~\citep{fql_park2025,zhan2026mean}. However, in the context of offline-to-online RL, one-step generation remains insufficient as the policy must improve beyond the offline behavioral prior. In addition, current one-step generative policies commonly rely on heuristic-based generate-and-select procedures, such as best-of-$N$ sampling, to bias actions toward high critic $Q$-values. Critically, this heuristic approach offsets the computational advantage of one-step policies by requiring many policy and critic evaluations per decision, while simultaneously not guaranteeing---for any finite $N$---\emph{optimal local improvement} of the action sampled under the critic.

\cut{

In contrast, we formalize one-step \emph{flow map policies} and propose an analytically optimal, closed-form method that aligns the action distribution with Q-value guidance, eliminating the need for BPTT, distillation networks, or best-of-$N$ heuristics.
}

\looseness=-1
\xhdr{Present work}
In this work, we introduce \emph{flow map policies}, a novel class of generative policies that learn the unique two-time jump operator associated with the probability flow ODE of diffusion and flow-matching policies. Crucially, flow map policies generalize existing one-step policies for offline-to-online RL, e.g., mean velocity policies~\citep{zhan2026mean}, while introducing new learning objectives yielding Lagrange, Euler, and Progressive variations of flow map policies. In stark contrast to prior work, for principled online adaptation of one-step flow map policies, we formulate a trust-region optimization problem and derive an analytically optimal, closed-form method that aligns the action distribution with $Q$-value guidance. This yields our contribution \namelong (\nameshort), 
which constructs a novel self-bootstrapped learning target---depicted in~\cref{fig:simple_side_by_side}---as the projected action-gradient of the critic, and eliminates the need for distillation networks or best-of-$N$ heuristics.


\looseness=-1
\looseness=-1
We additionally introduce a complementary inference-time search procedure at evaluation, \namelonginf (\nameshortinf) that combines stochastic sampling through renoising candidate samples to an intermediate state with $Q$-guided beam search around the trust region. Importantly, \nameshortinf produces diverse refinements around high $Q$-value actions without costly ODE simulations or additional learning during online adaptation. 
We summarize our core contributions as follows:

\begin{enumerate}[topsep=0pt, partopsep=0pt, itemsep=0pt, parsep=0pt, leftmargin=*]
    \item \looseness=-1 \xhdr{Flow map policies} We introduce flow map policies as a framework for learning one-step policies as two-time jump operators for flow-based generative actors.
    \item \looseness=-1 \xhdr{Algorithms} We introduce \nameshort, which efficiently adapts flow map actors using optimal $Q$-guidance (\Cref{thm:trust_region}) . We further introduce \nameshortinf, a stochastic inference-time refinement algorithm that combines flow map renoising, beam selection, and trust-region $Q$-guidance.
    \item \looseness=-1 \textbf{State-of-the-Art Performance:} 
    Across $12$ manipulation and locomotion tasks from OGBench and RoboMimic, \nameshort{} outperforms prior SOTA offline-to-online baselines by a relative average of $21.3\%$ while being on average $\approx 2.77 \times$ more efficient during online adaptation.
\end{enumerate}

\cut{
Diffusion and flow-matching generative models have emerged as powerful policy representations for reinforcement learning (RL), enabling expressive action distributions that Gaussian policies cannot capture~\citep{chi2023diffusion}. However, this expressiveness comes at a computational cost: sampling requires numerical simulation of a diffusion-policy SDE or flow-policy ODE, limiting its efficiency in real-world deployment.

To address this bottleneck, we leverage the flow map operator~\citep{boffi2024flow,flowmaps}---the unique jump operator associated with the flow-policy ODE---dramatically speeding up inference. In particular, we introduce \emph{flow map policies} that enable efficient one-step deployment by taking large jumps directly from noise to the physical action. 

Recent work in offline-to-online RL trains mean flow policies—a special case of flow maps—by relying on best-of-N sampling to simply select the highest Q-value action at each step \citep{mvp}. However, best-of-$N$ scales poorly due to requiring $N$ forward passes per action, and provides no mathematical guarantee that the selected action is optimal for any finite $N$. To overcome this, we formulate the problem of steering the action distribution to maximize Q-values as a trust-region-constrained optimization problem. By linearizing the Q-function, we derive an analytically optimal, closed-form steering step that maximizes the Q-value subject to a trust-region constraint on the current action. Building on this theoretical contribution, we propose \namelong (\nameshort), a framework for reward alignment in flow map policies for offline-to-online RL. Then, we introduce two novel, orthogonal mechanisms for reward alignment that utilize our derived closed-form steering step: a training-time fine-tuning method and an inference-time steering search algorithm.


\looseness=-1
\xhdr{Main contributions}
Our proposed training-time method fine-tunes the actor flow map to capture the optimally steered offline policy directly into its weights, while preventing catastrophic unlearning by anchoring trust region to the offline behavioral prior. During training, we dynamically scale the trust region based on the critic's epistemic uncertainty, allowing the policy to take aggressive optimization steps in well-explored regions while remaining highly conservative in uncertain states. Second, we introduce an inference-time search algorithm that iteratively steers a population of optimally steered action candidates using stochastic renoising. By partially corrupting candidate actions and denoising them via the flow map, the agent can efficiently explore distinct modes of the action distribution with high Q-values. We evaluate \nameshort on 12 continuous control tasks from OGBench~\citep{park2024ogbench} and RoboMimic, spanning complex manipulation and navigation domains.
}




\section{Background and Preliminaries}
\label{sec:background}

\looseness=-1
\xhdr{Offline-to-Online RL}
We consider a Markov Decision Process (MDP)~\citep{sutton1998reinforcement} defined by the tuple $\mathcal{M} = (\mathcal{S}, \mathcal{A}, r, P, \gamma)$, where $\mathcal{S} \subseteq \mathbb{R}^n $, $\mathcal{A} \subseteq \mathbb{R}^d$ denote continuous state-action spaces, $r(s,a)$ the reward function, $P(s'|s,a)$ the transition probability distribution, and $\gamma \in [0, 1)$ the discount factor. The objective of reinforcement learning is to train a policy $\pi(a|s)$ that maximizes the expected cumulative discounted return, $J(\pi) = \mathbb{E}_\pi \big[ \sum_{\tau=0}^\infty \gamma^{\tau} r(s^{\tau}, a^{\tau}) \big]$, where $\tau$ denotes the timestep. Offline-to-online RL is a two-stage learning framework consisting of offline pre-training followed by online fine-tuning. In the offline pre-training phase, a behavioral prior policy is trained on a static dataset $\mathcal{D} = \{(s, a, r, s')\}$, providing an initialization. Subsequently, during online fine-tuning, the policy directly interacts with the environment. A popular approach is actor-critic methods, which employ an actor $\pi(a \mid s)$ and a critic $Q_\theta(s,a)$ that approximates the expected discounted return under policy $\pi$: $Q^\pi(s,a) = \mathbb{E}_{\pi,P} \left[ \sum_{i=0}^\infty \gamma^i r(s^i, a^i) \;\middle|\; s^0 = s, a^0 = a \right]$. 

\looseness=-1
\xhdr{Flow matching polices}
\label{sec:generative_policies}
A generative policy $a_1 \sim \pi_{1}( \cdot | s)$ conditioned on state $s$, can be formulated as transport plan which pushes forward an easy to sample reference measure $p_0(a_0) \in \gP(\R^d)$ to a desired measure of (optimal) target action $p_1 (a_1):=  p^*_{\text{target}}(a^*) \in \gP(\R^d)$. The subscripts are indicative of a notion of time where the process evolves a (pseudo)-action from the prior at time $t=0$, i.e. $a_0 \sim p_0$, to an action that follows the target distribution $a_1 \sim p_1(a_1)$ at time $t=1$. We highlight that the time $t$ associated with the transport dynamics is distinguished from $\tau$, which is associated with the MDP. Formally, a \emph{flow-policy} is a one-parameter diffeomorphism, conditioned on a state $s$,  $\psi_{t}(\cdot | s):[0,1] \times \R^{d} \to \R^d$ that is the solution to the following ordinary differential equation:
\begin{equation}
     \frac{d}{dt} \psi_{t}(a_t \mid s) = v_{t} \left(\psi_{t}(a_t \mid s) \right), \quad \psi_0(a_0) = a_0,
     \label{eq:flow_policy_ode}
\end{equation}
\looseness=-1
with initial conditions $\psi_0(a_0) = a_0$. Furthermore, $v_t:[0,1]\times\R^d\to\R^{d}$ is a time-dependent (instantaneous) velocity field. 
In effect, the thesis of the generative policy problem is to learn a policy that pushes forward the base measure as follows $\pi_{1}(\cdot | s) := p_1 = [\psi_{1}( \cdot  | s)]_{\#}(p_0)$. We highlight that $\psi_1$ produces a \emph{deterministic} action while $\pi_1(\cdot | s)$ is induced \emph{distribution} over actions at $t=1$ by the flow-policy. 
To build the flow-policy, we can associate it with a conditional probability path $p_t( \cdot | z): [0, 1] \times \gP(\R^d) \to \gP(\R^d)$ which is a time-indexed interpolation in probability space between two distributions $p_0$ and $p_1$. In its simplest form, the conditioning variable can be taken to the endpoints $z:=(a_0, a_1)$, and a particle level interpolation that is simply a convex combination of the endpoints can be employed, i.e., $a_t = (1-t)a_0 + ta_1$. We say $\psi_t$ generates $p_t$ if it pushes forward $p_0$ to $p_1$ by following the ODE in~\cref{eq:flow_policy_ode}. To learn the flow policy, it is easier to regress against a known target conditional velocity $v^*_{t} (a_t | z,s)$ field that generates $p_t$. With access to such a $v^*_t$, learning can proceed using the conditional flow-matching loss~\citep{tong2023improving,stochasticinterpolants,liu2022flow,lipman2022flow,peluchetti2023non}, which is a simple simulation-free regression objective:
\begin{equation*}
\gL_{\rm CFM} = \mathbb{E}_{t, q(z), p_t(a_t | z,s)} \|v_{t}( a_t \mid s) - v^*_t(a_t \mid z,s)\|_2^2 = \mathbb{E}_{t, q(z), p_t(a_t \mid z,s)} \|v_{t}(a_t \mid s) - (a_1 - a_0)\|_2^2,
\label{eqn:CFM}
\end{equation*}
\looseness=-1
where $q(z)$ is a coupling over the states---e.g., independent coupling $q(z) = p_0(a_0) p_1(a_1)$---and in the last equality we substitute $v^*(a_t \mid z,s) = a_1 - a_0$ with its analytic linear speed target velocity. 

\cut{
We adopt the flow matching framework to learn a continuous transformation from a base distribution $p_{\text{init}} = \mathcal{N}(0, I)$ to the data distribution $p_{\text{data}}$. We define a conditional probability path by interpolating between $x_0 \sim p_{\text{init}}$ and $x_1 \sim p_{\text{data}}$:
\begin{equation}
x_t = c_t x_1 + s_t x_0, \quad x \sim p_t(\cdot \mid x_1) = \mathcal{N}(x; c_t x_1, s_t^2 I).
\end{equation}

For each $x_1 \sim p_{\text{data}}$, let $v_t(\cdot \mid x_1)$ denote a conditional vector field such that the corresponding ODE induces the target conditional probability path:
\begin{equation}
x_0 \sim p_{\text{init}}, \qquad 
\frac{d}{dt} x_t = v_t(x_t \mid x_1)
\;\;\Rightarrow\;\;
x_t \sim p_t(\cdot \mid x_1), \quad t \in [0,1].
\end{equation}

Flow matching learns a parametric vector field $\hat{v}_t^\theta(x)$ by minimizing the marginal flow matching loss:
\begin{equation}
\mathcal{L}_{\text{FM}}(\theta)
= \mathbb{E}_{t \sim \mathrm{Unif}[0,1], \; x \sim p_t}
\left[
\| \hat{v}_t^\theta(x) - v_t(x) \|^2
\right].
\end{equation}

Since the marginal objective is intractable, we instead optimize the equivalent conditional flow matching objective:
\begin{equation}
\label{eq:cfm}
\mathcal{L}_{\text{CFM}}(\theta)
= \mathbb{E}_{t \sim \mathrm{Unif}[0,1], \; x_1 \sim p_{\text{data}}, \; x \sim p_t(\cdot \mid x_1)}
\left[
\| \hat{v}_t^\theta(x) - v_t(x \mid x_1) \|^2
\right].
\end{equation}
}

\section{\namelong}
\label{sec:method}

\looseness=-1
Generative policies operate over an inner-time axis, i.e., ODE simulation time, that is distinct from the evolution of the MDP. Consequently, for every state along the trajectory $s^{\tau}$, a corresponding action $a^{\tau}_1$ must be generated through numerical simulation of the flow-policy ODE in~\cref{eq:flow_policy_ode}---necessitating large amounts of function evaluations of the policy network. We next address this computational inefficiency through learning the flow-map, which dramatically speeds up simulation by taking large jumps along the ODE trajectory. We organize the remainder of this section as follows: in~\S\ref{sec:flow_map_policies} we introduce \emph{flow-map policies} and apply them to offline RL in~\S\ref{sec:phase1}.
In~\S\ref{sec:online_adaptation} we rigorously design an efficient online adaptation update using a trust-region based on the critic's $Q$-function.
Finally, in~\S\ref{sec:inference_steering} we introduce our stochastic sampling approach to refine generated actions at inference.

\looseness=-1
\subsection{Flow Map Policies}
\label{sec:flow_map_policies}
\looseness=-1
For high-fidelity action generation using flow-matching policies, it remains critical to simulate the infinitesimal dynamics of the parametrized velocity field in~\cref{eq:flow_policy_ode}. Instead of solving the ODE, we can parametrize and learn the unique two-time operator associated with the flow-matching policy. 


\begin{mdframed}[style=MyFrame2]
\begin{definition}[Flow-Map Policy]
\label{def:flow_map_policy}
Let $X_{r,t}: [0,1]^2 \times \mathcal{S} \times \mathbb{R}^d \to \mathbb{R}^d$ be a flow map that evolves the action dynamics between any $(r,t) \in [0,1]$, conditioned on the MDP state $s \in \mathcal{S}$ governed by~\cref{eq:flow_policy_ode}, and satisfying the jump condition $X_{r,t}(a_r|s) = a_t$. The flow-map policy is then the distribution induced by this map evaluated at time $t=1$, where $ \pi(a | s) =  [X_{r,1}]_{\#} p_r(a_r|s)$.
\end{definition}
\end{mdframed}
\looseness=-1
To parametrize the underlying flow map that induces this policy, we leverage the \emph{average action velocity} $u_{r,t}: [0,1]^2 \times \mathcal{S} \times \mathbb{R}^d \to \mathbb{R}^d$, between the two time points $r,t$ with the condition $r \leq t$:
\begin{equation}
    X_{r,t}(a_r \mid s) = a_r + (t-r) u_{r,t}(a_r \mid s), \quad u_{r,t} (a_r \mid s) = \frac{1}{t-r}\int^t_r v_\tau(a_\tau \mid s) d\tau.
    \label{eq:flow_map parametrization}
\end{equation}
\looseness=-1
We note that, using~\cref{eq:flow_map parametrization}, we take jumps of size $t-r$ along the ODE trajectory. 
Furthermore, evaluating this flow map at the boundaries $r=0$ and $t=1$ yields the one-step policy $X_{0,1}$. 
We also highlight that the instantaneous velocity corresponds to the flow-matching policy can be recovered by taking the time limit yielding the tangent condition: $\underset{r\to t}{\lim} \partial_t X_{r,t}(a_r \mid s) = u_{t,t}(a_t|s) := v_t(a_t \mid s)$. As a result, this allows a supervision signal along the time diagonal $r=t$ amounting to classical flow-matching,
\begin{equation}
    \gL_{\text{Diag}} = \mathbb{E}_{t, q(z), p_t(a_t \mid  z,s)} \|u_{t,t}(a_t \mid s) - (a_1 - a_0)\|_2^2.
    \label{eq:diag_loss}
\end{equation}
\looseness=-1
To train the underlying flow map on the off-diagonal $r<t$, we follow the standard practice of enforcing consistency rules that are derived from satisfying the flow-map jump condition, as well as the semi-group property of the ODE~\citep{flowmaps}.  This leads to PINN style $\ell_2$-regression objectives that \emph{distill} the approximated ODE velocity field into $X_{r,t}$ by
enforcing the Lagrange, Euler, and Progressive conditions of the flow map on the off-diagonal times $r <t$ combined with a stop-gradient $\texttt{sg}$:

\begin{mdframed}[style=MyFrame2]
\begin{enumerate}[topsep=0pt, partopsep=0pt, itemsep=0pt, parsep=0pt, leftmargin=*]
    \item \xhdr{Lagrangian policy-distillation}
    \begin{equation}
        \label{eqn:lsd}
       \gL_{\text{LPD}} =  \int_0^1\int_0^t\E_{p_r(a_r \mid z,s)}\left[\left|\partial_t X_{r, t}(a_r \mid s) - \texttt{sg}\left(u_{t, t}(X_{r, t}(a_r \mid s))\right)\right|^2\right]dr dt,
    \end{equation}
    \item \xhdr{Eulerian policy-distillation}
    \begin{equation}
        \label{eqn:esd}
       \gL_{\text{EPD}} =  \int_0^1 \int_0^t\E_{p_r(a_r \mid z,s)}\left[\left|\partial_r X_{r, t}(a_r \mid s) + \texttt{sg}\left(\nabla X_{r, t}(a_r \mid s)u_{r, r}(a_r \mid s) \right)\right|^2\right]dr dt,
    \end{equation}
    \item \xhdr{Progressive policy-distillation}
    \begin{equation}
       \label{eqn:psd} 
       \begin{aligned}
        & \gL_{\text{PPD}} = \int_0^1\int_0^t\E_{p_r(a_r \mid z,s)}\left[\left|X_{w, t}(X_{r,w}(a_r \mid s)) - \texttt{sg}\left(X_{r,t}(a_r \mid s)\right)\right|^2\right]  dr dt d\gamma,
       \end{aligned}
    \end{equation}
    where $w = (1-\gamma) r + \gamma t$ with $\gamma \in [0, 1]$.
\end{enumerate} 
\end{mdframed}

\looseness=-1
\xhdr{Relation to mean-flow policies}
Critically, in contrast to prior work, equating policy learning with flow-maps unlocks the entire arsenal of flow-map-based learning objectives---with mean flows policies~\citep{meanflows,nguyenone,zhan2026mean} being a specific instantiation of the Eulerian policy. In particular, mean-flow policies can be derived as a specific instance of the Eulerian policy distillation objective outlined in~\cref{eqn:esd} above (see~\S\ref{app:equivalence}), and the instantaneous velocity constraint is an application of the tangent condition and is simply the diagonal loss in~\cref{eq:diag_loss}.

\cut{
We then exploit flow-map policies and design \namelong (\nameshort), 
a two-phase algorithm for offline-to-online reinforcement learning with flow map policies. In phase one, we pre-train an efficient (stochastic) flow-map actor through bootstrapped \emph{self-distillation} objectives. In phase two, we demonstrate a series of novel mechanisms to fine-tune this actor online by steering the flow-map policy towards high-$Q$ regions via a novel trust-region loss.
}

\cut{
\joey{Christos version}

As numerically integrating over an ODE is expensive, the flow map $X_{s,t}$ is a mapping from time $s$ to time $t$ along that trajectory. Formally, it is defined by the jump from the ODE state at time $s$ to time $t$ along the ODE trajectory:
\begin{equation}
X_{r,t}(x_r) = x_t.
\end{equation}

The tangent condition connects the parameterized flow map to the underlying ODE by taking the limit:
\begin{equation}
\lim_{r \to t} \partial_t X_{r,t}(x_r) = u_{t,t}(x),
\end{equation}
where $u: [0,1]^2 \times \mathbb{R}^d \to \mathbb{R}^d$ is the average velocity of the trajectory between times $s$ and $t$. In practice, we learn a neural parameterization $\hat{u}_{r,u}^\theta$, along with its induced flow map $\hat{X}_{r,u}^\theta$ that enforces the tangent condition along the diagonal $r = u$ via the flow matching objective~\ref{eq:cfm}:
\begin{equation}
\mathcal{L}_b(\theta) = \int_0^1 \mathbb{E} \left[ \left\| \hat{u}_{t,t}^\theta(x_t) - \dot{x}_t \right\|^2 \right] dt.
\end{equation}

To learn the flow map $X_{r,t}$, it remains to estimate $u_{r,t}$ away from the diagonal $r = t$. To this end, we leverage results that relate the off-diagonal $u_{r,t}$ to the instantaneous field $u_{t,t}$. One off-diagonal loss is the Eulerian self-distillation framework:
\begin{equation}
\mathcal{L}_{\mathrm{ESD}}(\theta) = \int_0^1 \int_0^t \mathbb{E}_{x_0, x_1} \left[ \left\| \partial_s \hat{X}_{r,t}^\theta(x_r) + \nabla \hat{X}_{r,t}^\theta(x_r)\hat{u}_{r,r}^\theta(x_r) \right\|^2 \right] dr \, dt.
\end{equation}

In sampling, the flow map operator $\hat{X}_{s,t}$ is evaluated via the Euler parameterization
\begin{equation}
\label{eq:sd}
\hat{X}_{r,t}^\theta(x_s) = x + (t - r)\hat{u}_{r,t}^\theta(x_r),
\end{equation}

where $\hat{u}^\theta$ is the minimizer of $u_{r,t}(x)$ for this self-distillation objective:
\begin{equation}
\mathcal{L}_{\mathrm{SD}}(\theta) = \mathcal{L}_b(\hat{v}) + \mathcal{L}_{ESD}(\theta).
\end{equation}

Other examples of self-distillation objectives within this framework include the Lagrangian and Progressive losses. In addition, this diagonal component can be provided by a frozen, pre-trained teacher model.
}
\subsection{Offline RL with Flow Map Policies}
\label{sec:phase1}

\looseness=-1
We now deploy flow-map policies for offline-to-online reinforcement learning within an actor-critic framework. We first pre-train an efficient flow-map actor on an existing offline dataset $\mathcal{D} = \{(s, a_1, r, s')\}$, along with a critic $Q$-network. We parametrize the actor as $u_{r,t}(a_r|s)$ over all time pairs $(r,t) \in [0,1]^2$, trained with the policy self-distillation objectives from~\S\ref{sec:flow_map_policies}: $\mathcal{L}_{\mathrm{actor}}^{\mathrm{off}} = \mathcal{L}_{\text{Diag}} + \lambda\, \mathcal{L}_{\text{SD}}$, where $\gL_{\text{SD}}$ corresponds to any of the policy self-distillation losses and $\lambda$ is a hyper-parameter that controls the strength of off-diagonal training. For maximally efficient action generation, we can simply invoke the flow-map policy $a_1 = X^{\text{off}}_{0,1}(a_0|s)$ and generate actions in a single forward pass by directly transporting the prior noisy action $a_0$ to the clean action $a_1$ in one step:
\begin{equation}\label{eq:one_step_inference}
a_1 = a_0 + u^{\text{off}}_{0,1}(a_0 \mid s), \quad a_0 \sim \mathcal{N}(0, I).
\end{equation}
\looseness=-1
We train the critics via clipped double $Q$-learning with EMA targets $Q_{{\phi}_j}$~\citep{fujimoto2018doubleqlearning}:
\begin{equation}\label{eq:critic_loss}
\mathcal{L}_{\text{critic}}(\phi_j) = \mathbb{E}_{(s, a_1, r, s') \sim \mathcal{D}} \left[ \left( Q_{\phi_j}(s, a_1) - y \right)^2 \right], \quad
y = r + \gamma \min_{j=1,2} Q_{\phi_j}(s', X^{\text{off}}_{0,1}( a_0' \mid s')).
\end{equation}

\subsection{Efficient Trust-Region Based Online Adaptation}
\label{sec:online_adaptation}

\looseness=-1
\cut{Transitioning from an offline to an online phase introduces the significant technical hurdle of identifying the optimal action to imitate when training generative policies. Indeed, in continuous action spaces without a curated dataset $\gD$ like in the offline phase, exhaustively searching for the optimal action $a^*$ as a learning target is all but computationally infeasible. A common strategy, instead, is to leverage the critic's learned $Q$-function to find critic-guided imitation targets. Implementing such a thesis falls under the paradigm termed ``generate-and-select", with the simplest instantiation being a ``best-of-$N$" sampler that greedily picks the best action under $Q$-function out of $N$ draws from the generative policy~\citep{zhan2026mean}. However, this naive strategy imposes a non-trivial drawback of requiring a large number of sampled actions $N$, which in the best-case scenario requires $N$ one-step simulations of the flow-map actor and $Q$-function evaluations. We next develop a more principled approach that finds the optimal action---without sampling $N$ actions---by constructing a trust region.}


\cut{
The optimization problem in~\cref{eq:nonlinear_trust_region_displacement_problem} pinpoints that starting from a noisy action state $a_r$ 
we must modify the base offline average velocity endpoint prediction $u_{r,1}(a_r|s)$ in a manner that maximizes the critics $Q$-value at the terminal time $t=1$. We further highlight that the trust-region constraint~\cref{eq:nonlinear_trust_region_displacement_problem} can be rewritten from an $\eta$-radius over $a_1$ to a constraint over the average velocity as both the offline flow map actor and the variable of optimization $\bar{u}_{r,1}$ share the same starting state $a_r$.
}

\looseness=-1
Transitioning to the online phase introduces the challenge of identifying the optimal action to imitate when training generative policies: in continuous action spaces without a curated dataset, solving $\arg\max_a Q(s,a)$ as a learning target is intractable. A common strategy is the ``best-of-$N$'' heuristic, which draws $N$ actions from the policy and selects the one with the highest $Q$-value~\citep{zhan2026mean}. However, this naive strategy imposes a non-trivial drawback of requiring a large number of sampled actions $N$, which requires at minimum $N$ one-step simulations of the flow-map actor and $Q$-function evaluations. We next develop a more principled approach that finds the optimal action by constructing a trust region. Consider a flow-map policy $\pi^{\text{off}}(\cdot|s)$, the natural question for online adaptation is:
\begin{center}
\vspace{-5pt}
    \textit{\textbf{Q.} What is the optimal perturbation $\Delta$ for $a_1 \sim \pi^{\text{off}}$ that maximizes the critic's $Q$-function?}
\vspace{-5pt}
\end{center}
\looseness=-1
To answer this question, we assume the existence of an optimal action $a_1^* = a_1 + \Delta^*$ that is feasible---i.e., reachable from the flow-map policy via a perturbation $\Delta$. To prevent unbounded deviation from $a_1 \sim \pi^{\text{off}}(\cdot | s)$, we constrain $\Delta$ within a trust region of radius $\eta$ around the critic's current $Q$-value. This yields the following non-linear optimization problem that maximizes the critic's $Q$-function:
\begin{equation}
\label{eq:nonlinear_trust_region_displacement_problem}
\begin{aligned}
 \argmax_{\Delta} \quad & \mathbb{E}_{r\sim \gU[0,1)} \left[Q_\phi\!\left(s, X^{\text{off}}_{0,r}(a_0 \mid s) + \Delta \right)\right] \quad \text{s.t.}
\quad
& \left\| \Delta \right\|_2
\leq \eta
\end{aligned}
\end{equation}
\looseness=-1
 In the case where the $\Delta$-perturbation is given as the average velocity network $u_{r,1}(a_r |s)$, constraining the perturbation $\|\Delta\|_2 \leq \eta$ in action space is equivalent to bounding $\|u_{r,1}(a_r|s) - u^{\text{off}}_{r,1}(a_r|s)\|_2$. As the critic $Q$-function is non-linear, this optimization problem is challenging to solve in closed form. Instead, we can consider a first-order approximation of optimality that aims to find \emph{optimal target displacement} to any generic reference $u^{\text{ref}}_{r,1}(a_r|s)$. Interestingly, under these settings, the analytic expression of the optimal target displacement admits a closed-form expression.

\begin{mdframed}[style=MyFrame2]
\begin{restatable}{theorem}{theoremone}
\label{thm:trust_region}
\looseness=-1
Consider a flow-map policy $\pi^{\text{ref}}(\cdot|s)$ with underlying flow map $X^{\text{ref}}_{r,1}$, generating actions $a_1 = a_r + (1-r)\,u^{\text{ref}}_{r,1}(a_r \mid s)$. The optimal average velocity $u^*_{r,1}$ that maximizes the first-order expansion of $Q_\phi$ around $a_1$, subject to trust-region constraint $\|u_{r,1} - u^{\text{ref}}_{r,1}\|_2 \le \eta$, is:
\begin{equation}
u^*_{r,1}(a_r \mid s) = u^{\text{ref}}_{r,1}(a_r \mid s) + \eta\, \frac{\nabla_a Q_\phi(s, a_1)}{\|\nabla_a Q_\phi(s, a_1)\|_2}.
\end{equation}
\end{restatable}
\end{mdframed}
\looseness=-1
\begin{proof}[Proof Sketch]
\looseness=-1 
To maximize $Q_\phi$ over $u_{r,1}$, we substitute the flow-map parameterization into a first-order Taylor expansion around $a_1$, which reduces the problem to maximizing $\langle \nabla_a Q_\phi(s, a_1),\, u_{r,1} - u^{\mathrm{ref}}_{r,1} \rangle$ subject to $\|u_{r,1} - u^{\mathrm{ref}}_{r,1}\|_2 \leq \eta$. Solving the associated KKT conditions yields the optimal closed-form solution. The full proof is provided in~\S\ref{app:proofs}.
\end{proof}
\looseness=-1
\Cref{thm:trust_region} holds for any reference flow-map policy. Setting $\pi^{\text{ref}} = \pi^{\text{off}}$, i.e., anchoring to the offline flow-map velocity $u^{\text{off}}_{r,1}(a_r|s)$, results in the following optimal average velocity:
\begin{equation}\label{eq:optimal_offline_avg_velocity}
u^*_{r,1}(a_r|s) = u^{\text{off}}_{r,1}(a_r|s) + \eta\, \frac{\nabla_a Q_\phi(s, a_1)}{\|\nabla_a Q_\phi(s, a_1)\|_2}.
\end{equation}
\looseness=-1
The analytic form of~\cref{eq:optimal_offline_avg_velocity} enables us to form a learning target for efficient online adaptation that we term \namelong (\nameshorttr). Specifically, we construct the interpolant $a_r = (1{-}r)\,a_0 + r\,a_{\text{data}}$ using noise $a_0 \sim \mathcal{N}(0,I)$ and actions from a replay buffer $a_{\text{data}} \sim \mathcal{D}$. This allows to then regress $u_{r,1}(a_r|s)$ against the optimal self-bootstrapped trust-region target below:
\begin{equation}\label{eq:fmq_f_loss}
\mathcal{L}_{\mathrm{FMQ}}(\theta) = \mathbb{E}_{r,\, a_0,\, a_{\text{data}}} \left[ \left\| u^{\theta}_{r,1}(a_r) - \mathrm{sg}\!\left( u^{\mathrm{off}}_{r,1}(a_r) + \eta \frac{\nabla_a Q_\phi(s, a_1)}{\|\nabla_a Q_\phi(s, a_1)\|_2 + \kappa_1} \right) \right\|_2^2 \right],
\end{equation}
\looseness=-1
where $\mathrm{sg}(\cdot)$ is the stop-gradient operator, and $\kappa_1 > 0$ is a stability constant as described in~\cref{alg:fmq_training}.


\looseness=-1
\xhdr{Uncertainty-Aware Adaptive Trust Region}
A fixed radius $\eta$ applies the same step size regardless of critic's confidence.
We now formulate an adaptive per-sample $\eta$-radius driven by a cheap heuristic, driven by capturing the epistemic uncertainty in the critic ensemble. Given a twin-critic ensemble, we define 
$\delta_{\text{critic}}(s,a) = \tfrac{1}{\sqrt{2}}|Q_{\phi_1}(s,a) - Q_{\phi_2}(s,a)|$ that captures the absolute discrepancy of $Q$-values amongst the critics. This allows us to design a batch-normalized per-sample \emph{effective} trust region,
\begin{equation}\label{eq:eta_eff}
  \eta_{\mathrm{eff}}(s, a)
    = \frac{1}{1 + \beta\,\tilde{\delta}_{\text{critic}}(s, a)},
  \qquad
  \tilde{\delta}_{\text{critic}}(s, a)
    = \frac{\delta_{\text{critic}}(s,a)}
           {\tfrac{1}{B}\sum_{i=1}^{B}\delta_{\text{critic}}(s_i, a_i) + \kappa_2},
\end{equation}
\looseness=-1
where $\kappa_2 > 0$ is a small constant, and $\beta$ a hyper-parameter.
By construction, $\eta_{\mathrm{eff}} \in (0,1]$ decays monotonically with the magnitude $\delta_{\text{critic}}(s,a)$: a small discrepancy $\delta_{\text{critic}}$ leads to larger steps, while conversely a larger discrepancy $\delta_{\text{critic}}$ results in the prioritization of the offline flow map actor. 

\subsection{Inference-Time $Q$-Guided Search}
\label{sec:inference_steering}
\looseness=-1
The flow map policy induces a mapping that transports noisy actions marginals $[X_{r,t}]_{\#}p_r = p_t$ for all $r,t \in [0,1]$. This mapping is fundamentally incapable of capturing the conditional posterior over endpoints that also maximize a critic's $Q$-value. As a result, the initial sampling of $a_0$ may have a disproportionate impact on the solutions to the optimization problem in~\cref{eq:nonlinear_trust_region_displacement_problem}. Instead of training a separate stochastic flow map for reward alignment~\citep{potaptchik2026meta,holderrieth2026diamond} we opt for a purely inference-time search strategy. Specifically, we next construct a stochastic sampler for flow map actors that also leverages the trust region of the critic's $Q$-value.


\looseness=-1
\xhdr{Stochastic Sampling with SNR}
To design a stochastic sampler, we leverage a renoising strategy based on the signal-to-noise ratio (SNR). In particular, 
given the one-step flow map actor \emph{after online adaptation} $a_1 = X^{\text{adapt}}_{0,1}(a_0 | s)$, we can re-noise by judiciously selecting a new time $t' < 1$. To do so, we design the re-noising interpolant by selecting $t' = \mathrm{SNR}/(1{+}\mathrm{SNR}) \in (0,1)$:
\begin{equation}\label{eq:renoise}
a_{t'} = t' \cdot a_1 + (1 - t') \cdot \varepsilon,
\quad \varepsilon \sim \mathcal{N}(0, I),
\end{equation}
\looseness=-1
A second application of the flow map then transports this intermediate state back to time $t=1$, i.e.,
$\tilde{a}_1 = X^{\text{adapt}}_{t',1}(a_{t'} \mid s)$. Crucially, each draw
of $\varepsilon$ yields a \emph{different} action
$\tilde{a}_1$---a stochastic sample from the flow map to which the trust-region update can be re-applied. The approach thus defines an iterative
refinement procedure during inference: (1) we corrupt the
current actions $a_1$ to $t'$ using~\cref{eq:renoise}, gaining access to diverse actions
noisy intermediate states $a_t'$, and (2) then re-apply the optimal trust-region $Q$-value projection to
each new sample, obtaining $\tilde{a}_1$. The noise level $t'$ controls the
exploration--exploitation balance: small $t'$ (low SNR) places
$a_{t'}$ closer to pure noise, allowing the flow map to explore
distant modes. Conversely, a large $t'$ (high SNR) preserves most of the
current action $a_1$, it restricts the update to a more local refinement.

\looseness=-1
\xhdr{$Q$-Guided Beam Search}
We now outline an inference-time search strategy that combines the stochastic sampler with beam search. 
This new algorithm is deployed only once at inference — i.e., inference-time scaling via search — and, as a result, does not affect training speed for online updates. We provide the full algorithmic description in~\cref{alg:diamond-steer}.
We instantiate this new final inference procedure \namelonginf (\nameshortinf), which balances exploration and exploitation. Specifically, \nameshortinf operates
over $M$ particles that are refined over $K$ steps along the already online-adapted flow map policy trajectory. In summary, the algorithm follows the following two steps iteratively:
\begin{enumerate}[topsep=0pt, partopsep=0pt, itemsep=0pt, parsep=0pt, leftmargin=*]
    \item \looseness=-1 \emph{\underline{Exploration:}} The first step in \nameshortinf constitutes an exploration phase that diversifies the candidate $N$ actions to a batch $B$ of intermediate states that then yields a total of $\bigcup_{i=1}^{N \cdot B} X^{\text{adapt}}_{t', 1} (a_t' |s)$ actions.
    \item \looseness=-1 \emph{\underline{Exploitation:}} The second step selects the most promising $M$ particles using the critic $Q_{\phi}(s, \tilde{a}_1)$, which are then used in the trust region update (\cref{eq:optimal_offline_avg_velocity}), before progressing to the next beam.
\end{enumerate}
\looseness=-1
After $K$ steps, the procedure returns $\arg\max_i Q_\phi(s, a_i)$. 
When $K{=}0$, the method reduces to best-of-$M$ Q-Steering (a single
application of Theorem~\ref{thm:trust_region} without iteration).


\section{Experiments}
\label{sec:experiments}
\begin{figure}[t]
\centering
\includegraphics[width=1.0\linewidth]{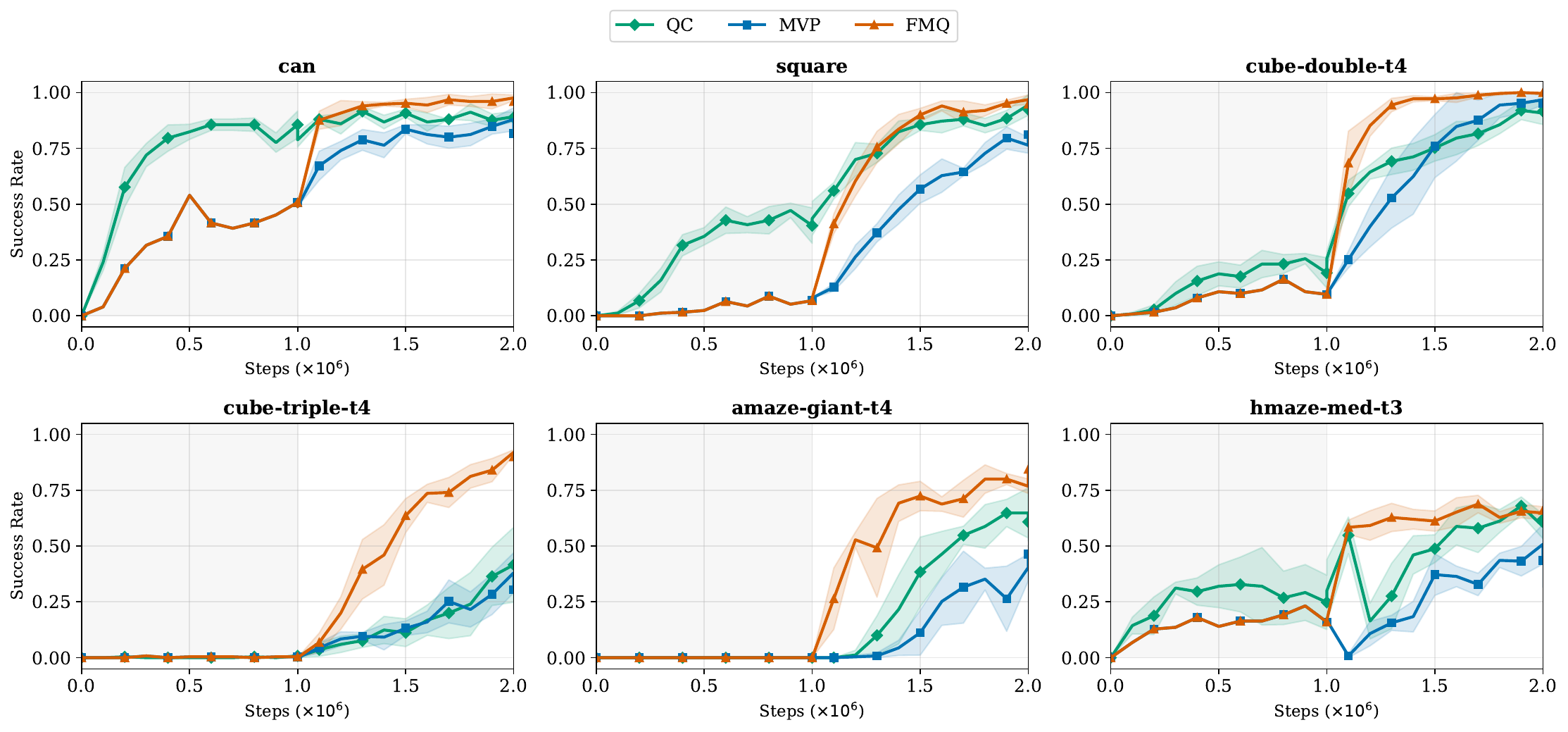}
\vspace{-20pt}
\caption{\small Training curves on 6 environments. Average success rate at every $100\mathrm{K}$ during 1M offline followed by 1M online steps over 5 seeds. Shaded regions indicate 95\% CIs.}
\vspace{-10pt}
\label{fig:training_curves_main}
\end{figure}

\looseness=-1
We investigate the application of \nameshort across $12$ robotic manipulation and locomotion tasks with varying difficulties across $7$ environments from the OGBench~\citep{ogbench} and Robomimic~\citep{robomimic} benchmarks. The manipulation tasks include two from Robomimic (\texttt{can}, \texttt{square}) and six from OGBench (\texttt{cube-dbl-t3/4}, \texttt{cube-trl-t3/4}, \texttt{scene-t4/5}). Locomotion is evaluated on four OGBench tasks (\texttt{hmaze-med-t3/4}, \texttt{amaze-gnt-t4/5}). During offline pre-training, we use multi-human demonstration datasets for Robomimic and the default noisy-expert datasets (play-style and navigate) for OGBench. In addition, we utilize single-task OGBench variants for offline-to-online RL. The humanoid and ant maze tasks use sparse rewards, while all others use dense rewards. 
For clarity, we report full training configurations and experimental setups in~\S\ref{app:implementation}.



\looseness=-1
\xhdr{Baselines} We compare \nameshorttr{} against two main baselines: (1) QC~\citep{li2025reinforcement} trains a multi-step flow matching policy with $10$ integration steps. (2) 
As our second baseline, we report the state-of-the-art method MVP~\citep{zhan2026mean}, which trains a mean flow policy with an initial velocity constraint. For MVP, we distinguish MVP$^*$ as results taken directly from the original paper, which is only available in $6/12$ environments considered here, from our reproduction MVP, allowing us to investigate all considered environments. All baselines share the same model architecture and follow the same clipped double $Q$-learning algorithm~\citep{fujimoto2018doubleqlearning}. At inference time, QC, MVP, and our method \nameshorttr{} all select actions via best-of-32 sampling, choosing the action with the highest $Q$-value.


\looseness=-1
\xhdr{Evaluation protocol} We follow the standard offline-to-online protocol from~\citet{ogbench}: $1\mathrm{M}$ gradient steps of offline pre-training using the provided dataset, followed by $1\mathrm{M}$ steps of online fine-tuning with environment interaction. During the online phase, newly collected transitions are appended to the replay buffer. To monitor training, we evaluate the policy every $100\mathrm{K}$ steps over $50$ episodes with randomized initial states. Finally, we evaluate the last checkpoint across $50$ unseen test episodes per environment and compute the average success rate across $5$ seeds and report the Interquartile Mean (IQM) alongside 95\% stratified bootstrap confidence intervals~\citep{agarwal2021deep}.

\subsection{Main Results}
\begin{wrapfigure}{r}{0.45\textwidth} 
    \centering
        \vspace{-40pt}
    \includegraphics[width=\linewidth]{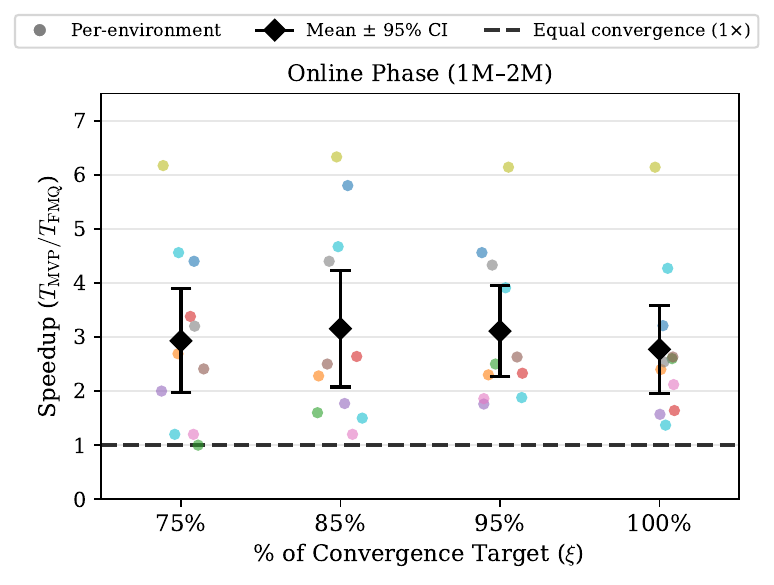} 
    \vspace{-20pt}
    \caption{\small \looseness=-1 Convergence speedup of \nameshorttr{} compared to MVP at success targets ($\xi$), with 95\% CIs.}
    \label{fig:speedupmain}
    \vspace{-5pt}
\end{wrapfigure}

\looseness=-1
We report our main quantitative results in~\cref{tab:final_results} and observe that \nameshorttr{} achieves the highest IQM score ($0.91;[0.89, 0.93]$), outperforms MVP by $21.3$\% ($0.75;[0.73, 0.77]$) and QC by $5.8$\% ($0.86;[0.84, 0.87]$). The improvement is most pronounced on challenging environments: on \texttt{cube-trl-t4}, \nameshorttr{} reaches $0.88$ compared to $0.37$ for QC and $0.32$ for MVP, and on \texttt{amaze-gnt-t4}, \nameshorttr{} achieves $0.80$ compared to $0.64$ and $0.42$, respectively. We also note that QC outperforms MVP on average ($0.76$ vs.\ $0.68$), but this comes at $10\times$ computational overhead at inference due to simulation of the flow rather than one-step generation. Nevertheless, we find that \nameshorttr{} outperforms QC in $10/12$ environments using only a single generation step. These results demonstrate the benefit of leveraging the optimal $Q$-guidance in \nameshorttr{} in comparison to best-of-$N$.

\looseness=-1
\xhdr{Inference scaling}
To evaluate our complementary contribution \namelonginf{} (\nameshortinf{}) that can be used as a stochastic sampler on any flow-map policy, including the baseline method MVP. We introduce two additional configurations, MVP + \nameshortinf{} and \nameshorttr{} + \nameshortinf{}. Specifically, we replace 
the  best-of-$N$ sampling in inference time with our stochastic sampling algorithm, which is combined with beam search ($K=1$, $B=4$, $M=4$) and outlined in~\S\ref{sec:inference_steering}. Overall, applying \nameshortinf we observe a relative increase in IQM by $8.0$\% (from $0.75$ to $0.81$) for MVP and a relative increase by $2.2$\% (from $0.91$ to $0.93$) for \nameshorttr{}. 
However, for a small number of locomotion tasks (\texttt{hmaze-med-t3} and \texttt{amaze-gnt-t4}), \nameshortinf{} degrades performance by $15.9$\% and $3.8$\% respectively. Overall, we find that combining both our proposed training and inference algorithms \nameshorttr{} + \nameshortinf{} leads to the best performance, achieving the highest IQM of $0.93;[0.91, 0.94]$, with non-overlapping confidence intervals against all baselines, including QC ($0.86;[0.84, 0.87]$) and MVP ($0.75;[0.73, 0.77]$). These results highlight the impact of performance by increasing the compute budget at inference through stochastic sampling and beam search.

\begin{table*}[t]
\centering
\footnotesize
\setlength{\tabcolsep}{3pt}
\caption{\small Success rate (mean $\pm$ std over 5 seeds, 50 episodes). Best per row in $\mathbf{bold}$, second best \underline{underlined}. Aggregate performance is measured by the IQM scores with 95\% CIs.}
\label{tab:final_results}
\resizebox{1\linewidth}{!}{
\begin{tabular}{lcccccc}
\toprule
\textbf{Environment} & \textbf{QC} & \textbf{MVP}$^*$ & \textbf{MVP} & \textbf{MVP + \nameshortinf (Ours)} & \textbf{\nameshorttr (Ours)} & \textbf{ \nameshorttr + \nameshortinf (Ours)} \\
\midrule
\texttt{can} & $0.88 \pm 0.06$ & $0.92 \pm 0.07$ & $0.83 \pm 0.07$ & $0.87 \pm 0.07$ & $\underline{0.96 \pm 0.04}$ & $\mathbf{0.97 \pm 0.03}$ \\
\texttt{square} & $0.89 \pm 0.04$ & $0.93 \pm 0.01$ & $0.82 \pm 0.04$ & $0.83 \pm 0.05$ & $\underline{0.94 \pm 0.02}$ & $\mathbf{0.95 \pm 0.04}$ \\
\texttt{cube-dbl-t3} & $\mathbf{1.00 \pm 0.00}$ & $\mathbf{1.00 \pm 0.00}$ & $\mathbf{1.00 \pm 0.00}$ & $\mathbf{1.00 \pm 0.00}$ & $\mathbf{1.00 \pm 0.00}$ & $\mathbf{1.00 \pm 0.00}$ \\
\texttt{cube-dbl-t4} & $0.92 \pm 0.05$ & $0.95 \pm 0.04$ & $\underline{0.98 \pm 0.02}$ & $\underline{0.98 \pm 0.02}$ & $\underline{0.98 \pm 0.02}$ & $\mathbf{1.00 \pm 0.00}$ \\
\texttt{cube-trl-t3} & $\underline{0.83 \pm 0.08}$ & $0.71 \pm 0.06$ & $0.64 \pm 0.12$ & $0.78 \pm 0.12$ & $0.78 \pm 0.10$ & $\mathbf{0.84 \pm 0.04}$ \\
\texttt{cube-trl-t4} & $0.37 \pm 0.26$ & $0.52 \pm 0.11$ & $0.32 \pm 0.07$ & $0.37 \pm 0.09$ & $\mathbf{0.88 \pm 0.07}$ & $\underline{0.87 \pm 0.05}$ \\
\texttt{scene-t4} & $\underline{0.99 \pm 0.01}$ & --- & $0.92 \pm 0.02$ & $0.98 \pm 0.02$ & $\mathbf{1.00 \pm 0.00}$ & $\underline{0.99 \pm 0.01}$ \\
\texttt{scene-t5} & $0.96 \pm 0.02$ & --- & $0.90 \pm 0.06$ & $0.95 \pm 0.05$ & $\underline{0.98 \pm 0.02}$ & $\mathbf{1.00 \pm 0.00}$ \\
\texttt{hmaze-med-t3} & $\underline{0.65 \pm 0.11}$ & --- & $0.47 \pm 0.10$ & $0.53 \pm 0.03$ & $\mathbf{0.69 \pm 0.04}$ & $0.58 \pm 0.07$ \\
\texttt{hmaze-med-t4} & $\underline{0.04 \pm 0.03}$ & --- & $0.00 \pm 0.00$ & $0.02 \pm 0.02$ & $\mathbf{0.06 \pm 0.03}$ & $\mathbf{0.06 \pm 0.03}$ \\
\texttt{amaze-gnt-t4} & $0.64 \pm 0.12$ & --- & $0.42 \pm 0.06$ & $0.43 \pm 0.04$ & $\mathbf{0.80 \pm 0.06}$ & $\underline{0.77 \pm 0.03}$ \\
\texttt{amaze-gnt-t5} & $\underline{0.91 \pm 0.05}$ & --- & $0.82 \pm 0.08$ & $0.90 \pm 0.06$ & $\mathbf{0.92 \pm 0.04}$ & $\mathbf{0.92 \pm 0.05}$ \\
\midrule
IQM SR\;[95\% CI]& $0.86\;[0.84, 0.87]$ & --- & $0.75\;[0.73, 0.77]$ & $0.81\;[0.78, 0.83]$ & $\underline{0.91\;[0.89, 0.93]}$ & $\mathbf{0.93\;[0.91, 0.94]}$ \\
\bottomrule
\end{tabular}
}
\vspace{-15pt}
\end{table*}


\subsection{Sample Efficiency Analysis}
\label{sec:experiments_sample_efficiency}

\looseness=-1
We next investigate the sample efficiency gains of using $Q$-guidance to train our flow map policies. In~\cref{fig:training_curves_main}, we plot the training curves during the online adaptation for all methods across $6$ environments (see~\cref{fig:training_curves_app} for full). We find that \nameshorttr{} consistently converges faster than MVP during online fine-tuning, despite both methods starting from the same offline checkpoint. To quantify this advantage, we define $\xi$ as the highest success rate that MVP and \nameshort{} can reach, computed per seed and environment. In~\cref{tab:time-to-threshold-combined}, we measure speedup $T$: the number of steps to first reach $\{75\%, 85\%, 95\%, 100\%\}$ of $\xi$. The speedup ratio $T_{\text{MVP}} / T_{\nameshorttr}$, averaged over 5 seeds, quantifies how many times faster \nameshorttr{} converges to each fraction of $\xi$. In the online phase (1M--2M), \nameshorttr{} reaches the highest success rate achievable by MVP $2.77\times$ faster on average, and up to $6.14\times$ on \texttt{hmaze-med-t3}.   These results further confirm that $Q$-gradient alignment provides a stronger learning signal than best-of-$N$ selection, leading to faster policy improvement per environment step.

\subsection{Ablation Studies}
\label{sec:ablation_studies}

\begin{wraptable}{r}{0.40\linewidth}
    \centering
    \vspace{-40 pt}
    \caption{ \small \nameshortinf efficiency ablation.}
    \label{tab:iqm_fe}
    \small
    \setlength{\tabcolsep}{3pt}
    \begin{tabular}{llcc}
    \toprule
    $K$ & $\{B, M\}$ & $\mathrm{NFE}$ & IQM \\
    \midrule
    $0$ & $\{1, 32\}$ & $32$ & $0.91\;[0.89,0.93]$ \\
    \midrule
    $1$ & $\{4, 4\}$ & $20$ & $\mathbf{0.93\;[0.91,0.94]}$ \\
    $1$ & $\{2, 8\}$ & $24$ & $\mathbf{0.93\;[0.91,0.95]}$ \\
    $1$ & $\{1, 16\}$ & $32$ & $0.92\;[0.90,0.93]$ \\
    $1$ & $\{4, 16\}$ & $80$ & $0.92\;[0.90,0.93]$ \\
    \midrule
    $2$ & $\{4, 4\}$ & $36$ & $0.91\;[0.90,0.93]$ \\
    $2$ & $\{1, 16\}$ & $48$ & $0.90\;[0.89,0.92]$ \\
    $2$ & $\{4, 16\}$ & $144$ & $0.90\;[0.89,0.91]$ \\
    \bottomrule
    \end{tabular}
    \vspace{-5pt}
\end{wraptable}

\looseness=-1
\textbf{Inference-time Beam Search.} 
The computational cost of utilizing \nameshortinf{} is $\text{NFE} = M(1 + KB)$ per action selection, where $M$ is the number of initial candidates, $K$ is the number of re-noising steps, and $B$ is the number of completions per candidate. We note that best-of-$N$ sampling corresponds to $K{=}0$ and $M{=}N$. In~\cref{tab:iqm_fe}, we show that the optimal configuration ($K{=}1$, $B{=}4$, $M{=}4$) achieves a peak IQM of $0.93$ with only $20$ FE---$37.5$\% fewer than best-of-$32$---suggesting that diversifying candidates through renoising is more efficient than considering more candidates. Increasing $K$ beyond $1$ does not improve performance, suggesting that only a modest increase in inference cost is needed for optimal performance.

\begin{figure}[t]
\centering
\vspace{-10pt}
\includegraphics[width=1.0\linewidth]{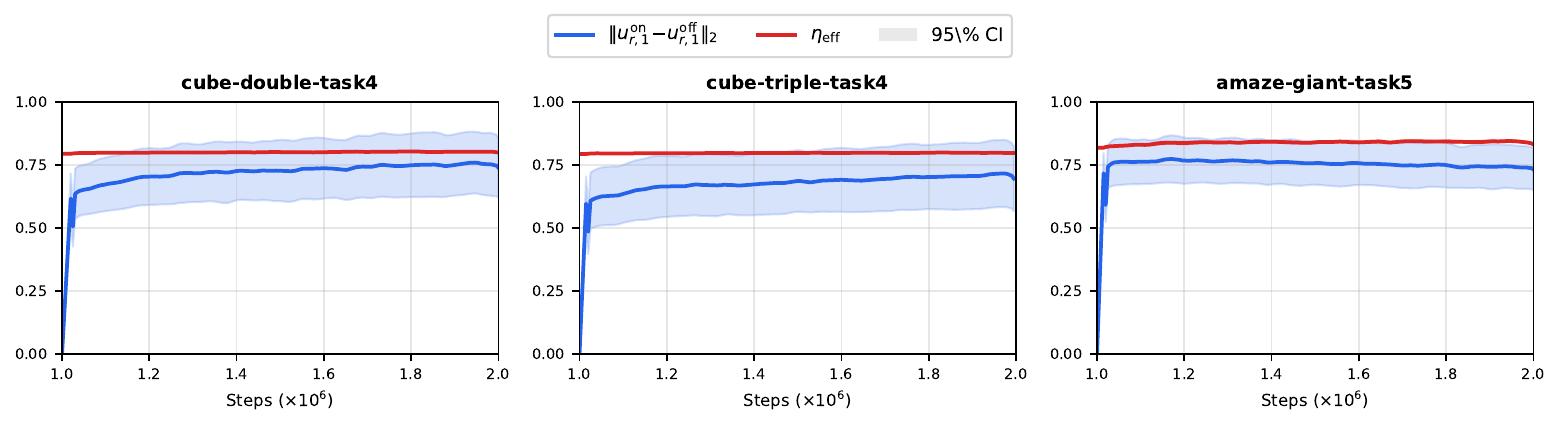}
\vspace{-25pt}
\caption{\small Convergence during online training on 3 environments. Distance between the online and offline flow map policies (blue) converges to the adaptive radius $\eta_{\mathrm{eff}}$ (red) as the policy incorporates the $Q$-guidance.}
\vspace{-15pt}
\label{fig:kkt_validation_main}
\end{figure}

\looseness=-1
\xhdr{Trust-Region Convergence Analysis} We evaluate the impact of trust-region  in~\cref{eq:fmq_f_loss}. Specifically, we measure the distance $\|u^{\text{on}}_{r,1} - u^{\text{off}}_{r,1}\|_2$ relative to the frozen offline policy. At the onset of online training, the online and offline velocity fields coincide, and the distance is $0$. As online training begins, the learning actively drives $u^{\text{on}}_{r,1}$ toward $u^{\text{off}}_{r,1} + \eta \nabla_a Q_\phi(s, a_1)/\|\nabla_a Q_\phi(s, a_1)\|_2$---as evidenced in~\cref{fig:kkt_validation_main}---and stabilizes near $\eta_{\text{eff}}$. Thus, the flow map policy incorporates the normalized $Q$-gradient direction while remaining safely constrained within the trust-region radius (c.f.~\cref{fig:kkt_validation_app} for all environments).

\looseness=-1
\xhdr{Flow map policy variants} 
We next ablate the offline flow map policy variants and their impact on performance. In~\cref{tab:ablation_sd_loss}. ESD and LSD achieve the same IQM of $0.79$, while PSD lags slightly at $0.77$. 
However, ESD exhibits the narrowest $95\%$ bootstrap CI $[0.77, 0.81]$, indicating it is the most consistent across environments. We therefore adopt ESD as the default formulation for our presented results.

\section{Related work}
\label{sec:related_work}

\looseness=-1
\xhdr{Generative Policies}
Diffusion models and flow matching have emerged as expressive policy representations~\citep{chi2023diffusion, pearce2023imitating}. For policy learning, prior methods train diffusion and flow-matching models via weighted behavioral cloning~\citep{lu2023qgpo, kang2023edp}, reparameterized policy gradients~\citep{wang2022diffusionql, ding2024consistencyac, zhang2024eql}, and rejection sampling~\citep{chen2023sfbc, hansen2023idql, he2024aligniql}. While effective, reparameterized gradients require costly backpropagation through time (BPTT). To address the latency of flow-matching multi-step models, FQL~\citep{fql_park2025} distills a multi-step flow into a separate one-step student network, while QC~\citep{li2025reinforcement} groups action sequences. MVP~\citep{zhan2026mean} natively achieves one-step generation but uses a ``generate-and-select'' heuristic to find imitation targets. In contrast, we formalize one-step flow map policies and 
leverage \nameshort for more efficient online adaptation.

\begin{wraptable}{r}{0.52\linewidth}
    \centering
    \vspace{-20pt}
    \caption{Self-distillation loss ablation.}
    \label{tab:ablation_sd_loss}
    \scriptsize
    \setlength{\tabcolsep}{2pt}
    \begin{tabular}{lccc}
    \toprule
    \textbf{Environment} & \textbf{ESD} & \textbf{PSD} & \textbf{LSD} \\
\midrule
\texttt{cube-trl-t3} & $0.79 \pm 0.06$ & $0.74 \pm 0.13$ & $\mathbf{0.84 \pm 0.09}$ \\
\texttt{cube-trl-t4} & $\mathbf{0.90 \pm 0.05}$ & $0.87 \pm 0.09$ & $0.88 \pm 0.06$ \\
\texttt{hmaze-med-t3}   & $\mathbf{0.64 \pm 0.02}$ & $0.48 \pm 0.29$ & $0.62 \pm 0.09$ \\
\texttt{hmaze-med-t4}   & $0.03 \pm 0.03$ & $0.04 \pm 0.04$ & $\mathbf{0.06 \pm 0.00}$ \\
\texttt{amaze-gnt-t4} & $0.82 \pm 0.02$ & $\mathbf{0.86 \pm 0.04}$ & $0.78 \pm 0.06$ \\
\texttt{amaze-gnt-t5} & $\mathbf{0.92 \pm 0.01}$ & $0.87 \pm 0.05$ & $0.91 \pm 0.03$ \\
\midrule
IQM     & $\mathbf{0.79}$\;[0.77, 0.81] & $0.77$\;[0.70, 0.81]  & $\mathbf{0.79}$\;[0.75, 0.82] \\
\bottomrule
\end{tabular}
\vspace{-10pt}
\end{wraptable} 

\looseness=-1
\textbf{Offline-to-Online RL.} 
Offline-to-online RL accelerates online learning by initializing from a static dataset~\citep{levine2020offline}. However, offline RL must contend with overestimation of $Q$-values for out-of-distribution actions, addressed through divergence penalties~\citep{fujimoto2019bcq, wu2019brac, nair2020awac, wang2022diffusionql}, pessimistic value estimates~\citep{kumar2020conservative, yu2020mopo, an2021uncertainty}, or in-sample maximization~\citep{kostrikov2021iql, garg2023xql}. When transitioning to the online phase, distribution shift can cause catastrophic forgetting of the behavioral prior~\citep{lee2021offline, song2023hybrid, nakamoto2023calql}. Recent state-of-the-art flow-matching methods rely on behavioral regularization~\citep{fql_park2025} or best-of-$N$ selection~\citep{li2025reinforcement, zhan2026mean} to stabilize adaptation. In contrast, we formulate online fine-tuning as a trust-region problem with \nameshort. 

\section{Conclusion}
\label{sec:conclusion}

\looseness=-1
In this paper, we bridge the gap between expressive generative policies and the low-latency requirements of offline-to-online RL. We formulate online adaptation of one-step flow map policies as a trust-region optimization problem, yielding \namelong{} (\nameshort{}): a closed-form, locally optimal $Q$-guided update that improves a linearized critic while remaining anchored to the offline behavioral prior. We further introduce \nameshortinf{}, an inference-time refinement procedure based on stochastic renoising and beam search that improves any flow map policy. Across $12$ continuous-control tasks from OGBench and RoboMimic, \nameshort{} establishes state-of-the-art performance, outperforming the previous leading one-step policy, MVP, by a $21.3\%$ relative margin in IQM success. While \nameshort enjoys an efficient linearized critic approximation, its effectiveness depends on critic accuracy and local linearity. Extending flow map adaptation with curvature-aware updates, stronger uncertainty estimates, and deployment on physical robotic platforms are promising directions for future work.

\cut{
In this paper, we bridge the gap between the expressiveness of generative policies and the inference latency requirements of offline-to-online RL. While prior one-step generative methods rely on best-of-$N$ heuristics, we formalized online adaptation as a trust-region-constrained optimization problem. This perspective allowed us to derive \namelong{} (\nameshort{}), an analytically optimal, closed-form steering step that maximizes a linearized Q-function while remaining safely anchored to the offline behavioral prior.  Building on this theoretical foundation, we proposed two orthogonal mechanisms: a training-time fine-tuning algorithm (\textsc{FMQ-FT}) that dynamically scales the trust-region radius using the critic's epistemic uncertainty to prevent catastrophic forgetting, and an inference-time search algorithm (\nameshortinf{}, or \textsc{Q-RBS}) that leverages stochastic renoising and beam search for robust action refinement. Empirically, our framework establishes a new state-of-the-art across 12 challenging continuous control tasks from OGBench and RoboMimic. Notably, \nameshort{} significantly accelerates sample efficiency during online fine-tuning and outperforms the previous leading one-step policy (MVP) by a $21.3\%$ relative margin in IQM success. While \nameshort{} provides an exact, mathematically optimal local Q-gradient step, its efficacy inherently relies on the accuracy of the learned critic and its first-order Taylor approximation. In environments with highly non-linear reward landscapes, this linear approximation may not be.  Furthermore, while our uncertainty-aware radius effectively prevents policy collapse, exploring principled ways to incorporate second-order Hessian information or learned local curvature metrics into the flow map update could safely enable larger optimal jumps. Finally, deploying flow map policies on physical robotic platforms is an interesting future direction.}

\section*{Acknowledgments}

We thank Jorge de Freitas for their helpful comments on an earlier draft and Luca Eyring for the insightful discussion on renoising. This work was supported by UKRI (EP/Y037111/1) as part of the ProSafe project (EU Horizon 2020, MSCA, grant no.\ 101119358).

\bibliographystyle{abbrvnat}
\bibliography{bibliography}
\newpage


\appendix

\section{Theoretical details}
\label{app:theory}

\subsection{Proofs}
\label{app:proofs}

\begin{mdframed}[style=MyFrame2]
\theoremone*
\end{mdframed}

\begin{proof}

We seek the optimal average velocity $u_{r,1}(a_r \mid s)$ that generates action
\begin{equation}
\bar{a}_1 = a_r + (1-r)\, u_{r,1}(a_r \mid s).
\end{equation}

Let $u_{r,1}(a_r \mid s)$ denote a candidate average velocity, generating action $\bar{a}_1= a_r + (1-r) u_{r,1}(a_r \mid s)$. We take the first-order Taylor expansion of the Q-function around the reference action $a_1$:
\begin{equation}\label{eq:taylor_expansion}
Q_\phi(s, \bar{a}_1) \approx Q_\phi(s, a_1) + \langle \nabla_a Q_\phi(s,\, a_1),\, \bar{a}_1 - a_1 \rangle
\end{equation}
Substituting the flow map parameterization, the starting state $a_r$ cancels out, leading to the difference in average velocity vectors:
\begin{equation}
\bar{a}_1 - a_1 = \bigl(a_r + (1-r)u_{r,1}(a_r \mid s)\bigr) - \bigl(a_r + (1-r)u^{\text{ref}}_{r,1}(a_r \mid s) \bigr) = (1-r)(u_{r,1}(a_r \mid s) - u^{\text{ref}}_{r,1}(a_r \mid s))
\end{equation}
Since $Q_\phi(s, a_1)$ is constant with respect to $u_{r,1}(a_r \mid s)$, maximizing the linear approximation subject to the trust-region constraint on the average velocity is formulated as:
\begin{equation}\label{eq:constrained_opt}
\begin{aligned}
\min_{u_{r,1}} \quad & f_0(u_{r,1}(a_r \mid s)) = -\langle \nabla_a Q_\phi(s,\, a_1),\, u_{r,1}(a_r \mid s) - u^{\text{ref}}_{r,1}(a_r \mid s) \rangle \\
\text{subject to} \quad & f_1(u_{r,1}(a_r \mid s)) = \tfrac{1}{2}\| u_{r,1}(a_r \mid s) - u^{\text{ref}}_{r,1}(a_r \mid s) \|_2^2 - \tfrac{1}{2}\eta^2 \le 0
\end{aligned}
\end{equation}
Because $\eta > 0$, the interior of the feasible set is non-empty, satisfying Slater's constraint qualification. Therefore, strong duality holds and the KKT conditions are necessary and sufficient. The Lagrangian is:
\begin{equation}
\begin{split}
\mathcal{L}(u_{r,1}, \lambda) &= -\langle \nabla_a Q_\phi(s,\, a_1),\, u_{r,1}(a_r \mid s) - u^{\text{ref}}_{r,1}(a_r \mid s) \rangle \\
&\quad + \lambda \left( \tfrac{1}{2}\| u_{r,1}(a_r \mid s) - u^{\text{ref}}_{r,1}(a_r \mid s) \|_2^2 - \tfrac{1}{2}\eta^2 \right)
\end{split}
\end{equation}

Let $u^*_{r,1}(a_r \mid s)$ and $\lambda^*$ be the primal and dual optima. The KKT conditions are:
\begin{align}
    f_1(u^*_{r,1}(a_r \mid s)) &\le 0 &&\text{(Primal feasibility)} \label{eq:kkt_primal} \\
    \lambda^* &\ge 0 &&\text{(Dual feasibility)} \label{eq:kkt_dual} \\
    \lambda^* f_1(u^*_{r,1}(a_r \mid s)) &= 0 &&\text{(Complementary slackness)} \label{eq:kkt_slackness} \\
    -\nabla_a Q_\phi(s,\, a_1) + \lambda^* (u^*_{r,1}(a_r \mid s) - u^{\text{ref}}_{r,1}(a_r \mid s)) &= 0 &&\text{(Stationarity)} \label{eq:kkt_stationarity_new}
\end{align}
Assuming $\nabla_a Q_\phi(s,\, a_1) \neq 0$, stationarity~\eqref{eq:kkt_stationarity_new} requires $\lambda^* > 0$. Complementary slackness~\eqref{eq:kkt_slackness} then forces the constraint to be active:
\begin{equation}\label{eq:active_constraint}
\| u^*_{r,1}(a_r \mid s) - u^{\text{ref}}_{r,1}(a_r \mid s) \|_2 = \eta
\end{equation}
Taking the norm of the stationarity condition gives $\lambda^* \eta = \|\nabla_a Q_\phi(s,\, a_1)\|_2$, so $\lambda^* = \|\nabla_a Q_\phi(s,\, a_1)\|_2 / \eta$. Substituting $\lambda^*$ back into the stationarity condition yields:
\begin{equation}\label{eq:final_proof}
u^*_{r,1}(a_r \mid s) = u^{\text{ref}}_{r,1}(a_r \mid s) + \eta \frac{\nabla_a Q_\phi(s, a_1)}{\|\nabla_a Q_\phi(s, a_1)\|_2}
\end{equation}
\end{proof}

\subsection{Equivalence of Eulerian and Mean Flow Policies}
\label{app:equivalence}

\looseness=-1
In this section, we elucidate the equivalence between Mean Flow Policies~\citep{zhan2026mean} and the Eulerian Policy in~\cref{eqn:esd}. 
We begin by stating the Mean Flow Policy and its loss gradient with respect to parameters $\theta$.
    \begin{align}
    \gL_{\text{MF}} &= \E \left[\left|u^{\theta}_{r,t}(a_r \mid s) - \texttt{sg}\left(u^{\theta}_{r,t}(a_r \mid s) + (t-r) \frac{\ddd}{\ddd r} u^{\theta}_{r,t}(a_r \mid s) \right)\right|^2\right]  \\
    \nabla_{\theta}\gL_{\text{MF}} &=  2\E \left[ \nabla_{\theta} u^{\theta}_{r,t}(a_r \mid s)^{T} \left( u^{\theta}_{r,t}(a_r \mid s) - \texttt{sg}\left(u^{\theta}_{r,t}(a_r \mid s) + (t-r) \frac{\ddd}{\ddd r} u^{\theta}_{r,t}(a_r \mid s) \right)\right)\right], \nonumber
    \end{align}
where the expectation is taken with respect to $(r,t, p_r(a_r \mid z, s))$.
Now let us recall the Eulerian objective with a flow map policy parametrization $X_{r,t}(a_r \mid s) = a_r + (t-r) u^{\theta}_{r,t}(a_r \mid s)$ with explicit parameters $\theta$ for the average velocity:
    \begin{equation}
       \gL_{\text{EPD}}(\theta) =  \E \left[\left|\partial_r X_{r, t}(a_r \mid s) + \texttt{sg}\left(\nabla X_{r, t}(a_r \mid s)u^{\theta}_{r, r}(a_r \mid s) \right)\right|^2\right],
    \end{equation}
\looseness=-1
Let us examine the terms inside the squared norm and remove the stop-gradient operator $\texttt{sg}$. We compute the partial derivative with respect to the start time $r$:
\begin{equation}
    \partial_r X_{r,t}(a_r \mid s) = -u^{\theta}_{r,t}(a_r \mid s) + (t-r) \partial_r u^{\theta}_{r,t}(a_r \mid s).
    \label{eq:partial_time} 
\end{equation}
\looseness=-1
Plugging this back into the Eulerian objective, we have,
    \begin{equation}
    \gL =  \E\left[\left| \underbrace{-u^{\theta}_{r,t}(a_r \mid s)}_{T_1} + \underbrace{(t-r) \partial_r u^{\theta}_{r,t}(a_r \mid s) + \nabla X_{r, t}(a_r \mid s)u^{\theta}_{r, r}(a_r \mid s)}_{T_2}\right|^2\right].
    \end{equation}
\looseness=-1
Applying a stop-gradient to $T_2$ and taking parameter gradients, 
Plugging this back into the Eulerian objective, we have,
    \begin{equation}
    \nabla_{\theta}\gL(\theta) =  2\E  \left[ u^{\theta}_{r,t}(a_r \mid s)  -\nabla_{\theta} u^{\theta}_{r,t}(a_r \mid s) \cdot\left( u^{\theta}_{r,t}(a_r \mid s) + \texttt{sg}\left(T_2\right)\right)\right].
    \label{eq:esd_loss_grad}
    \end{equation}
\looseness=-1
Now expanding the spatial gradient term in $T_2$, that is $\nabla X_{r,t}(a_r \mid s) u^{\theta}_{r, r}(a_r \mid s)$:
\begin{align}
    \nabla X_{r,t}(a_r \mid s)u^{\theta}_{r, r}(a_r \mid s) = u^{\theta}_{r, r}(a_r \mid s) + (t-r) \nabla u^{\theta}_{r,t}(a_r \mid s)u^{\theta}_{r, r}(a_r \mid s). \label{eq:spatial_grad}
\end{align}
\looseness=-1
Now by invoking the tangent condition and replacing $u^{\theta}_{r,r}$ with the ground truth instantaneous velocity $v^{*}$ we can expand $T_2$ have
    \begin{equation*}
   T_2 = (t-r) \partial_r u^{\theta}_{r,t}(a_r \mid s) + v^*_{r}(a_r \mid s) + (t-r) \nabla u^{\theta}_{r,t}(a_r \mid s)v^*_r(a_r \mid s).
    \end{equation*}
\looseness=-1
Rearranging terms and grouping $(t-r)$ terms, we notice the total derivative $\ddd / \ddd r$ corresponds exactly to $T_2$. We now leverage and rewrite~\cref{eq:esd_loss_grad} succinctly:
    \begin{equation*}
    \nabla_{\theta}\gL(\theta) =  2\E\left[\nabla_{\theta} u^{\theta}_{r,t}(a_r \mid s)^T \left(u^{\theta}_{r,t}(a_r \mid s)- \texttt{sg}\left(u^{\theta}_{r,t}(a_r \mid s) + (t-r) \frac{\ddd}{\ddd r} u_{r,t}(a_r \mid s) \right) \right)\right].
    \label{eq:esd_loss_grad_total_derivative}
    \end{equation*}

\looseness=-1
This loss gradient matches the Mean Flow Policies' loss gradient, with the main distinction being the usage of the ground truth velocity $v^*_r$ as opposed to the network's prediction $u^{\theta}_{r,r}$. Furthermore, the instantaneous velocity constraint is equivalent to the diagonal loss of~\cref{eq:diag_loss}. This demonstrates that Mean Flow policies~\citep{zhan2026mean} are not an independent paradigm, but rather a specific instantiation of the broader Eulerian Policy Distillation framework.

\newpage
\section{Successful Rollouts}
\label{app:env_viz}

We visualize successful rollouts from the trained \nameshorttr{} policy across all 12 evaluation environments. Each figure shows uniformly-spaced frames from a single episode that achieves the task goal.


\begin{figure}[htbp!]
    \centering
    \includegraphics[width=0.90\textwidth]{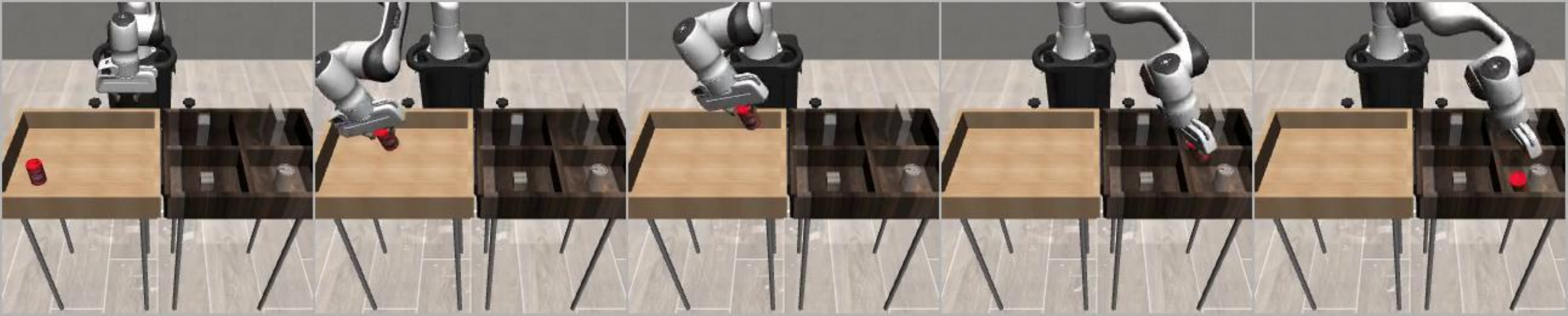}
    \caption{\texttt{can} (Robomimic). Pick up a can from the table and place it into the bin.}
    \label{fig:rollout-can}
\end{figure}
\looseness=-1
\begin{figure}[htbp!]
    \centering
    \includegraphics[width=0.90\textwidth]{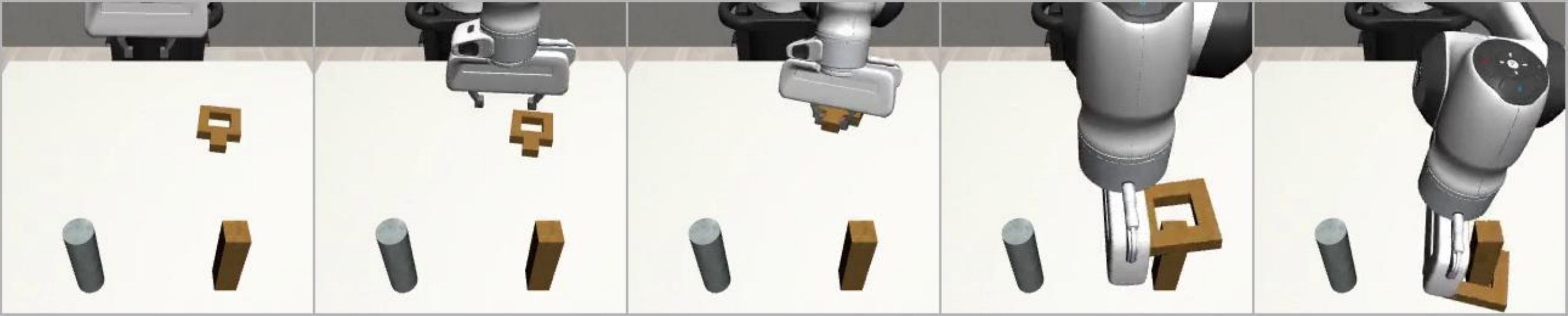}
    \caption{\texttt{square} (Robomimic). Pick up a square nut and fit it onto a peg.}
    \label{fig:rollout-square}
\end{figure}
\looseness=-1
\begin{figure}[htbp!]
    \centering
    \includegraphics[width=0.90\textwidth]{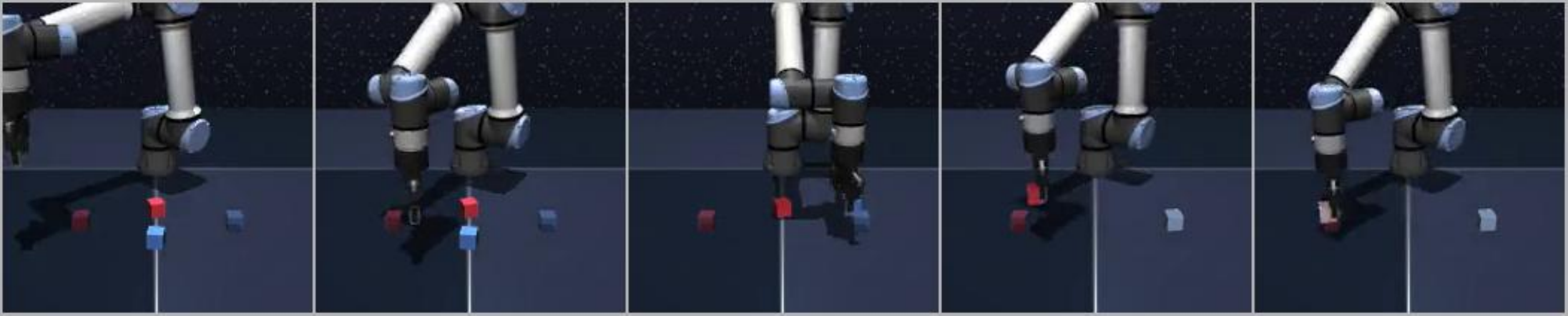}
    \caption{\texttt{cube-double-task3} (OGBench). Rearrange 2 cubes to target positions.}
    \label{fig:rollout-cdp3}
\end{figure}
\looseness=-1
\begin{figure}[htbp!]
    \centering
    \includegraphics[width=0.90\textwidth]{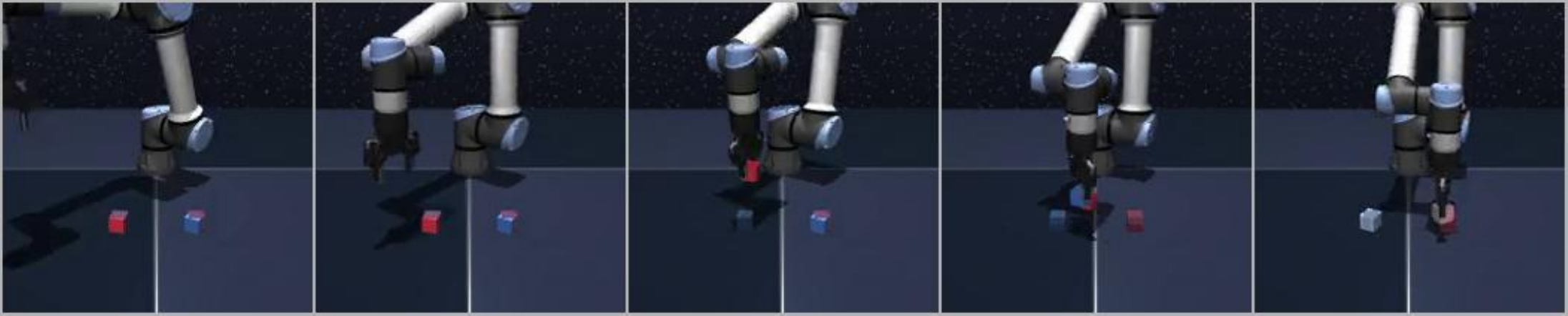}
    \caption{\texttt{cube-double-task4} (OGBench). Swap the positions of 2 cubes.}
    \label{fig:rollout-cdp4}
\end{figure}
\looseness=-1
\begin{figure}[htbp!]
    \centering
    \includegraphics[width=0.90\textwidth]{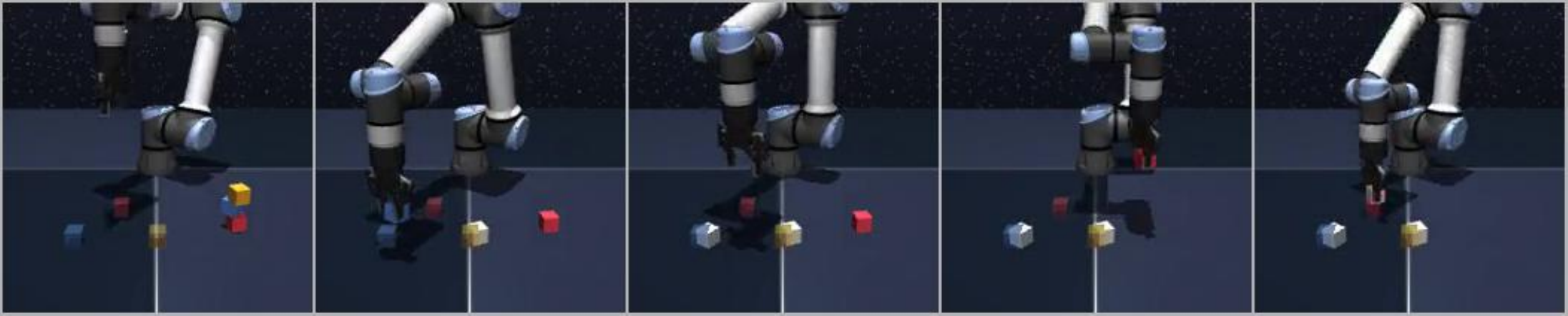}
    \caption{\texttt{cube-triple-task3} (OGBench). Unstack 3 cubes and place them at separate target positions.}
    \label{fig:rollout-ctrp3}
\end{figure}
\looseness=-1
\begin{figure}[htbp!]
    \centering
    \includegraphics[width=0.90\textwidth]{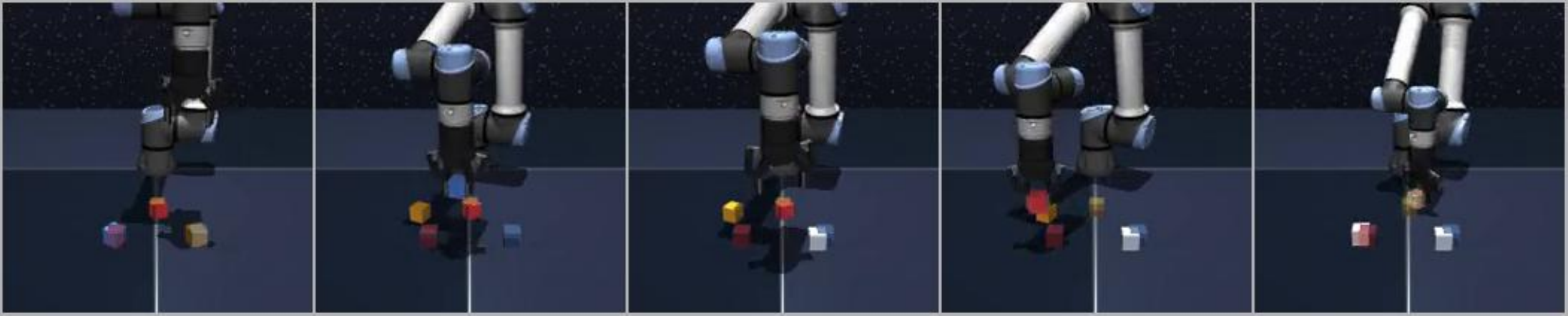}
    \caption{\texttt{cube-triple-task4} (OGBench). Cyclically permute 3 cubes to new positions.}
    \label{fig:rollout-ctrp4}
\end{figure}
\looseness=-1
\begin{figure}[htbp!]
    \centering
    \includegraphics[width=0.90\textwidth]{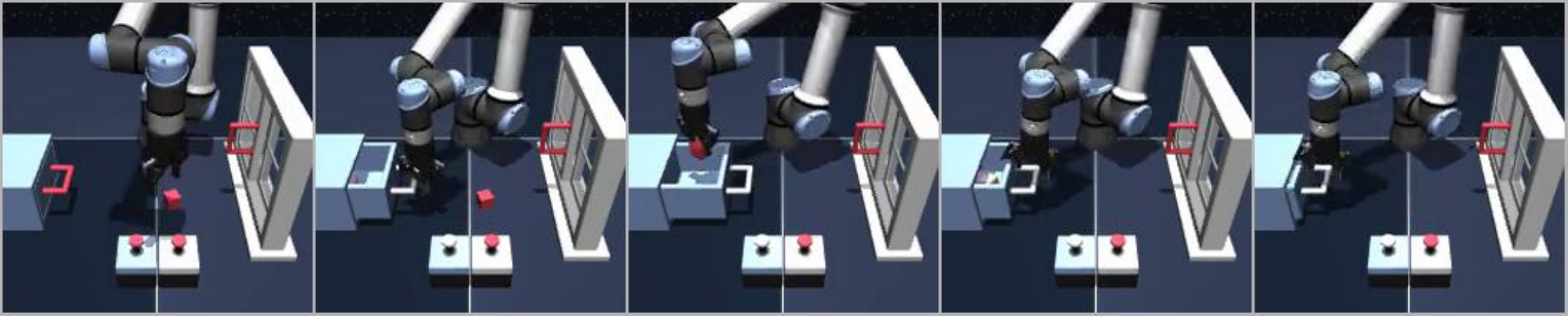}
    \caption{\texttt{scene-task4} (OGBench). Unlock the drawer button, open the drawer, and place the cube inside.}
    \label{fig:rollout-sc4}
\end{figure}
\looseness=-1
\begin{figure}[htbp!]
    \centering
    \includegraphics[width=0.90\textwidth]{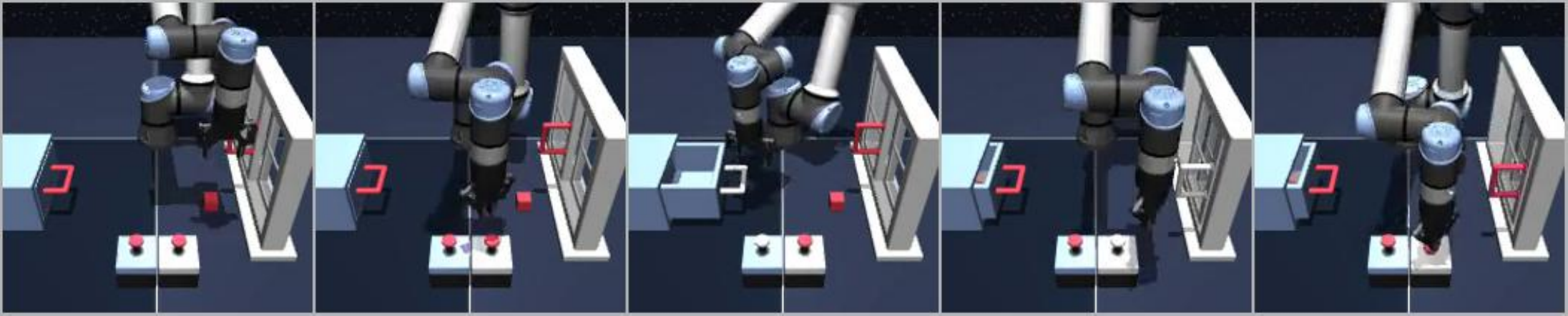}
    \caption{\texttt{scene-task5} (OGBench). Place the cube in the drawer and open the window.}
    \label{fig:rollout-sc5}
\end{figure}
\looseness=-1
\begin{figure}[htbp!]
    \centering
    \includegraphics[width=0.90\textwidth]{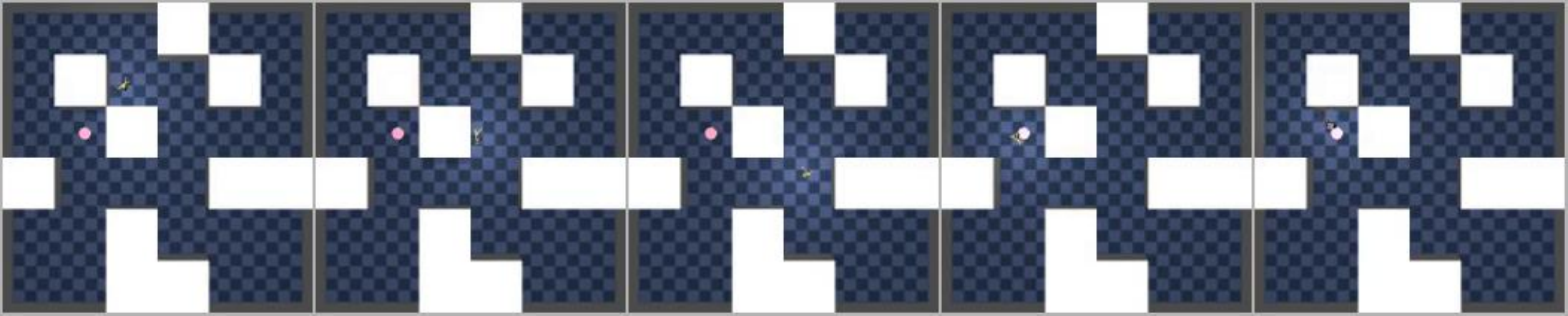}
    \caption{\texttt{humanoidmaze-medium-task3} (OGBench). Navigate a humanoid to a goal in a medium maze.}
    \label{fig:rollout-hm3}
\end{figure}
\looseness=-1
\begin{figure}[htbp!]
    \centering
    \includegraphics[width=0.90\textwidth]{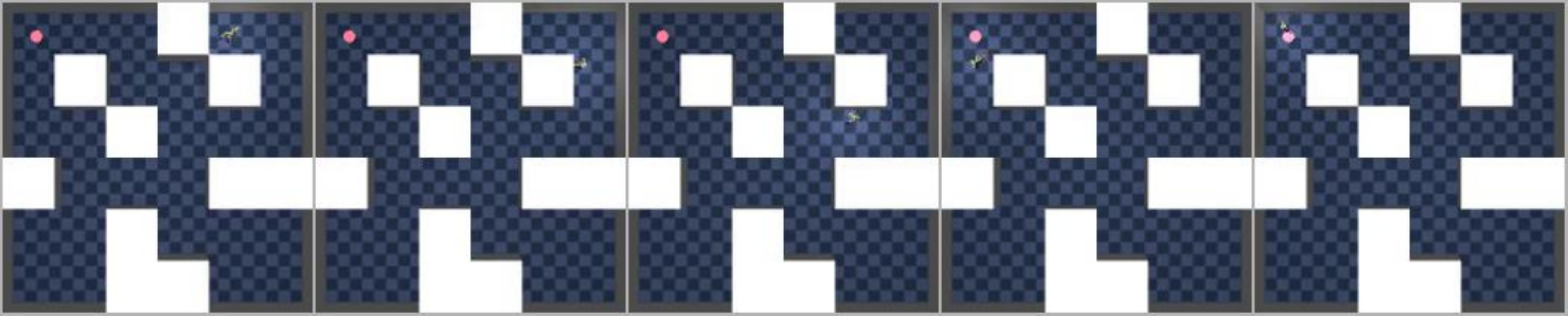}
    \caption{\texttt{humanoidmaze-medium-task4} (OGBench). Navigate a humanoid to a distant goal in a medium maze.}
    \label{fig:rollout-hm4}
\end{figure}
\looseness=-1
\begin{figure}[htbp!]
    \centering
    \includegraphics[width=0.90\textwidth]{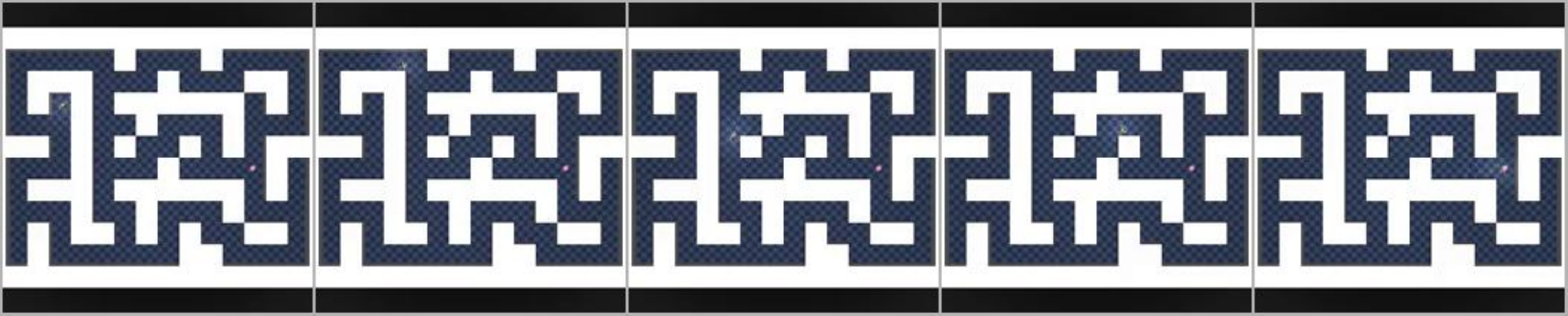}
    \caption{\texttt{antmaze-giant-task4} (OGBench). Navigate an ant to a goal across a giant maze.}
    \label{fig:rollout-ag4}
\end{figure}
\looseness=-1
\begin{figure}[htbp!]
    \centering
    \includegraphics[width=0.90\textwidth]{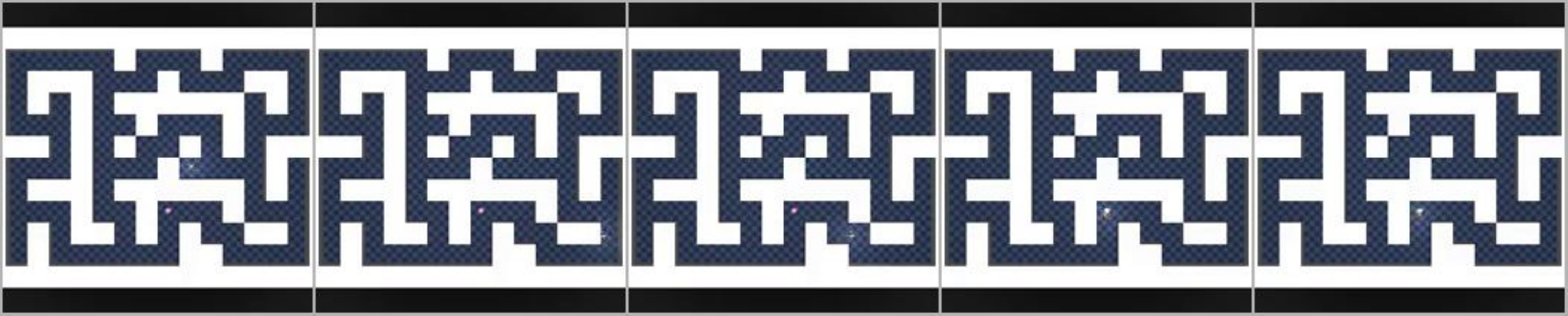}
    \caption{\texttt{antmaze-giant-task5} (OGBench). Navigate an ant to a nearby goal in a giant maze.}
    \label{fig:rollout-ag5}
\end{figure}
\looseness=-1

\newpage
\section{Algorithms}
\looseness=-1
\begin{algorithm}[h]
\caption{\namelong (\nameshorttr) }
\label{alg:fmq_training}
\begin{algorithmic}[1]
\REQUIRE Offline policy $u^{\text{off}}_{r,1}$, online policy $u^\theta_{r,1}$, critics $Q_{\phi_1}, Q_{\phi_2}$, buffer $\mathcal{D}$
\FOR{each environment step}
    \STATE $a_1 \gets a_0 + u^\theta_{0,1}(a_0|s)$, $a_0 \sim \mathcal{N}(0,I)$ 
    \STATE $\mathcal{D} \gets \mathcal{D} \cup \{(s, a_1, r, s')\}$
    \STATE Sample batch from $\mathcal{D}$; update critics via Eq.~\ref{eq:critic_loss}
    \STATE $r \sim \mathcal{U}[0,1)$; $a_0 \sim \mathcal{N}(0,I)$; $a_r \gets (1{-}r)a_0 + r\,a_{\text{data}}$
    \STATE $a_1 \gets a_r + (1{-}r)\,u^{\text{off}}_{r,1}(a_r|s)$
    \STATE $g \gets \nabla_a Q_{\phi_1}(s,a_1) / (\|\nabla_a Q_{\phi_1}(s,a_1)\|_2 + \kappa_1)$
    \STATE $\eta_{\text{eff}} \gets \eta / (1 + \beta\,\tilde{\delta}_{\text{critic}})$ \hfill $\triangleright$ Eq.~\ref{eq:eta_eff}
    \STATE $\theta \gets \theta - \alpha\,\nabla_\theta\|u^\theta_{r,1}(a_r|s) - \mathrm{sg}(u^{\text{off}}_{r,1}(a_r|s) + \eta_{\text{eff}}\,g)\|^2$
\ENDFOR
\end{algorithmic}
\end{algorithm}
\looseness=-1

\begin{algorithm}[h]
\caption{\namelonginf (\nameshortinf)}
\label{alg:diamond-steer}
\begin{algorithmic}[1]
\REQUIRE Flow map $X^\theta_{r,1}$, critic $Q_\phi$, state $s$, beam $M$, steps $K$, branches $B$, SNR $\rho$, step size $\eta$
\STATE $t' \gets \rho/(1{+}\rho)$
\STATE Sample $\{a_0^m\}_{m=1}^M \sim \mathcal{N}(0,I)$; \; $a_1^m \gets a_0^m + u^\theta_{0,1}(a_0^m|s)$ for all $m$
\FOR{$k = 1, \ldots, K$}
    \FOR{$m = 1, \ldots, M$ and $b = 1, \ldots, B$}
        \STATE $\varepsilon^{mb} \sim \mathcal{N}(0,I)$
        \STATE $\hat{a}_1^{mb} \gets X^\theta_{t',1}\!\left(t'\,a_1^m + (1{-}t')\,\varepsilon^{mb} \mid s\right)$ \hfill $\triangleright$ Re-noise \& complete
    \ENDFOR
    \STATE $\{a_1^m\}_{m=1}^M \gets \mathrm{Top\text{-}}M\!\left(\{\hat{a}_1^{mb}\}_{m,b};\; Q_\phi(s, \hat{a}_1^{mb})\right)$ \hfill $\triangleright$ Select best $M$ of $M{\cdot}B$
    \STATE $a_1^m \gets a_1^m + \eta\,\nabla_a Q_\phi(s, a_1^m) / \|\nabla_a Q_\phi(s, a_1^m)\|_2$ for all $m$ \hfill $\triangleright$ Thm.~\ref{thm:trust_region}
\ENDFOR
\STATE \textbf{return} $a_1^{\arg\max_m Q_\phi(s, a_1^m)}$
\end{algorithmic}
\end{algorithm}

\newpage

\newpage
\section{Training Curves}
\label{app:training_curves}

\Cref{fig:training_curves_app} extends~\cref{fig:training_curves_main} to all 12 environments. On the simpler manipulation tasks (\texttt{can}, \texttt{square}, \texttt{cube-dbl}), all methods converge to near-perfect success, but \nameshorttr{} reaches this level earlier. The advantage becomes more pronounced on the harder tasks: on \texttt{cube-trl-t4}, \nameshorttr{} reaches $0.88$ while MVP plateaus at $0.32$; on \texttt{amaze-gnt-t4}, \nameshorttr{} achieves $0.80$ versus $0.42$ for MVP. For locomotion (\texttt{hmaze}, \texttt{amaze}), the Q-gradient signal is particularly beneficial under sparse rewards, where best-of-$N$ selection provides a weaker learning signal.

\begin{figure}[t]
\centering
\includegraphics[width=1.0\linewidth]{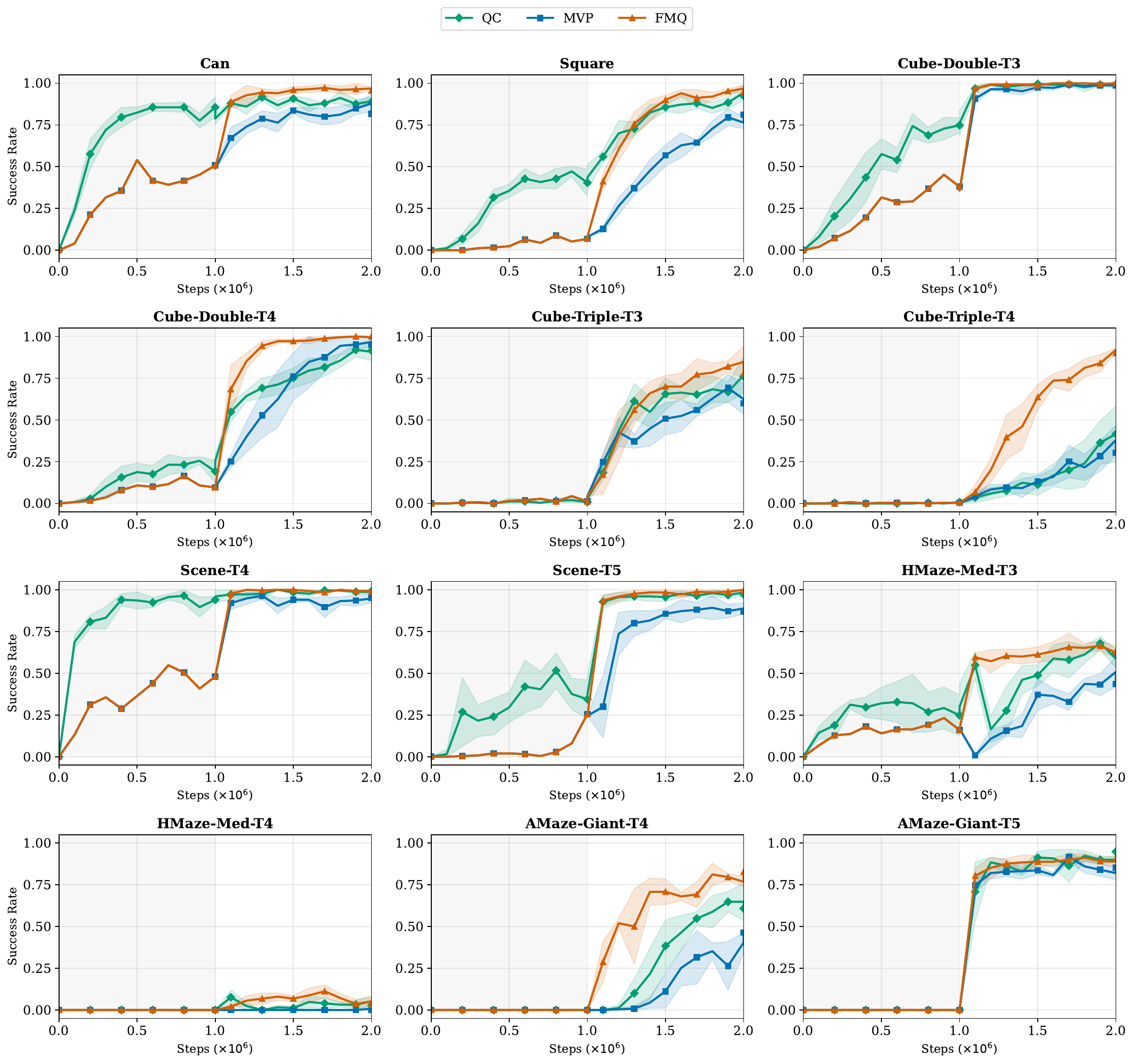}
\caption{Offline-to-online learning curves for QC, MVP, and \nameshorttr{} on all environments. All methods perform 1M offline followed by 1M online steps. Shaded regions indicate 95\% CIs over 5 seeds.}
\label{fig:training_curves_app}
\end{figure}

\section{Trust-Region Convergence}
\label{app:convergence}

\looseness=-1
\Cref{fig:kkt_validation_app} extends the convergence analysis of~\cref{fig:kkt_validation_main} to all 12 environments. We track the action displacement $\|u^{\text{on}}_{r,1} - u^{\text{off}}_{r,1}\|_2$ between the online and frozen offline flow map policies throughout online training. At the onset of fine-tuning ($1\text{M}$ steps), both policies coincide and the displacement is near zero. As training progresses, the trust-region loss in~\cref{eq:fmq_f_loss} drives $u^{\text{on}}_{r,1}$ toward $u^{\text{off}}_{r,1} + \eta_{\mathrm{eff}}\,\hat{g}$, causing the displacement to grow monotonically until it stabilizes near $\eta_{\mathrm{eff}}$. The orange curve (right axis) shows the implied Q-uncertainty $\tilde{\sigma}_Q = (\eta_{\mathrm{eff}}^{-1} - 1)/\beta$, which decreases as the critic becomes more confident---automatically tightening the trust region and confirming that the adaptive mechanism prevents overshooting in low-confidence regions.

\begin{figure}[t]
\centering
\includegraphics[width=1.0\linewidth]{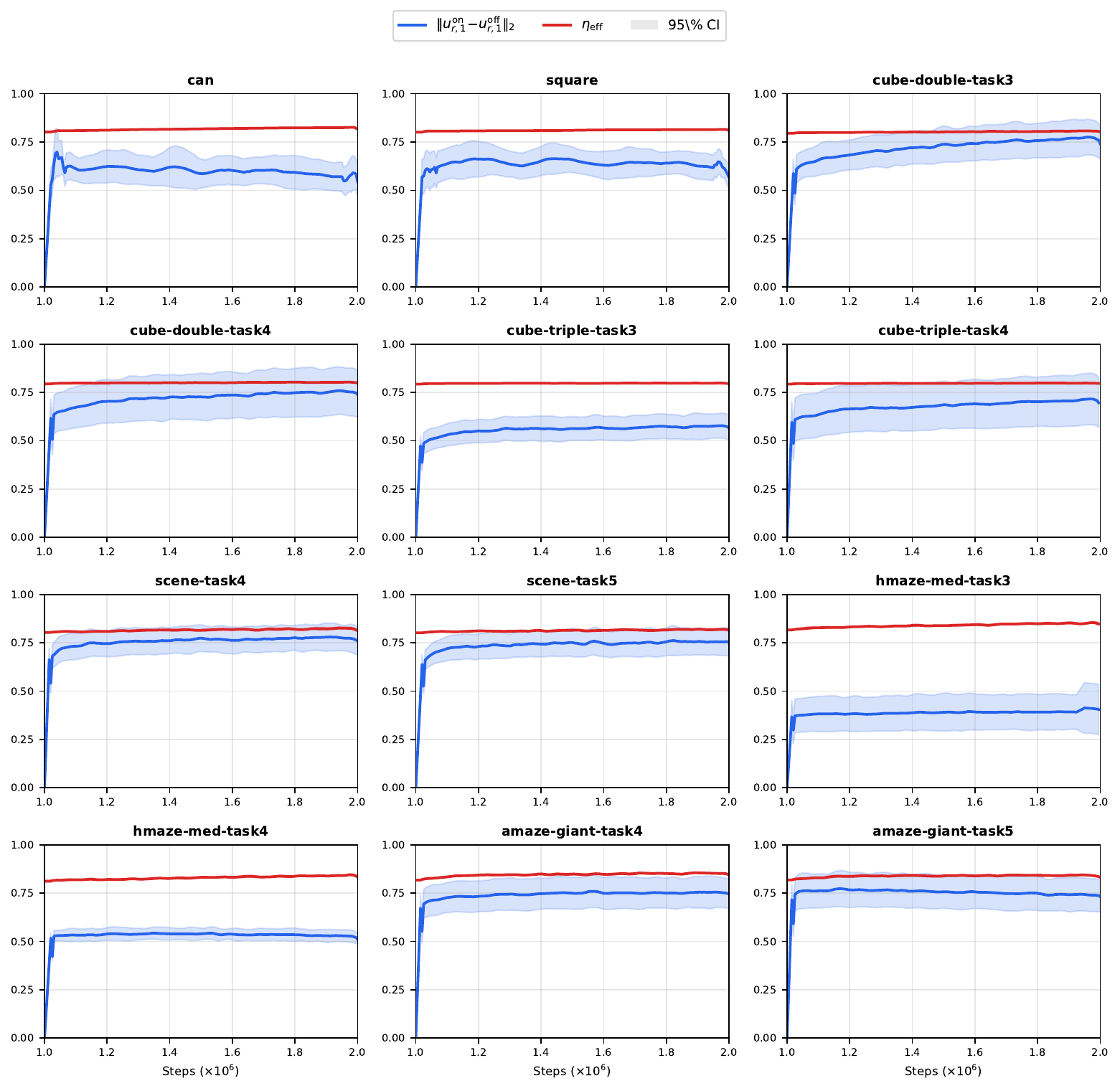}
\caption{%
Trust-region convergence for \nameshorttr{} ($\beta{=}0.3$) across all 12 environments. Blue: action displacement $\|u^{\text{on}}_{r,1} - u^{\text{off}}_{r,1}\|_2$. Red dashed: adaptive trust-region radius $\eta_{\mathrm{eff}}$. Orange dotted (right axis): implied Q-uncertainty $\tilde{\sigma}_Q = (\eta_{\mathrm{eff}}^{-1} - 1)/\beta$.}
\label{fig:kkt_validation_app}
\end{figure}

\newpage

\section{Speedup Analysis}
\label{app:speedup}

\looseness=-1
\Cref{fig:speedup} visualizes the per-environment convergence speedup of \nameshorttr{} over MVP during the online phase (1M--2M), complementing the discussion in~\cref{sec:experiments_sample_efficiency}. For each threshold $\xi \in {75\%, 85\%, 95\%, 100\%}$ of the shared convergence target, we plot the ratio $T_{\text{MVP}} / T_{\nameshorttr}$. \nameshorttr{} is faster than MVP on every environment at every threshold (all points above $1\times$), with average speedups of $2.8$--$3.2\times$. The full per-environment breakdown including both full-training and online-only phases is provided in~\cref{tab:time-to-threshold-combined}.
\begin{figure}[t]
\centering
\includegraphics[width=0.9\linewidth]{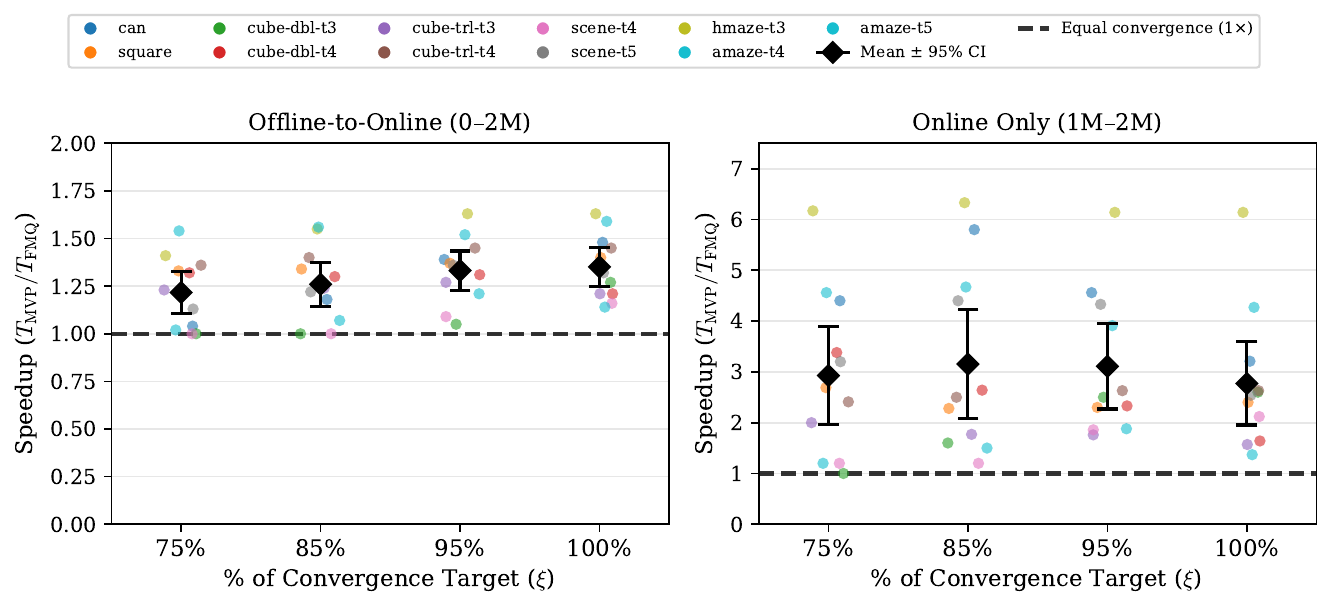}
\vspace{-15pt}
\caption{\looseness=-1 Convergence speedup of \nameshorttr{} over MVP ($T_{\text{MVP}} / T_{\nameshorttr}$) during online phase (1M--2M steps). Each dot represents one environment; black diamonds show the mean with 95\% CI. The dashed line marks equal convergence speed ($1\times$).}
\label{fig:speedup}
\end{figure}

\begin{table*}[t]
\centering
\tiny
\setlength{\tabcolsep}{4pt}
\caption{Speedup of \nameshorttr{} over MVP ($T_{\text{MVP}} / T_{\nameshorttr}$) measuring time to reach a fraction of the convergence target $\xi$ (per seed, averaged over 5 seeds). \textbf{Left}: full training (0--2M). \textbf{Right}: online phase only (1M--2M). Values $> 1$ indicate \nameshorttr{} is faster.}
\label{tab:time-to-threshold-combined}
\resizebox{1\linewidth}{!}{
\begin{tabular}{l c c c c c c c c}
\toprule
 & \multicolumn{4}{c}{\textbf{Full (0--2M)}} & \multicolumn{4}{c}{\textbf{Online (1M--2M)}} \\
\cmidrule(lr){2-5} \cmidrule(lr){6-9}
\textbf{Environment} & 75\% & 85\% & 95\% & 100\% & 75\% & 85\% & 95\% & 100\% \\
\midrule
\texttt{can} & $1.04$ & $1.18$ & $1.39$ & $1.48$ & $4.40$ & $5.80$ & $4.56$ & $3.21$ \\
\texttt{square} & $1.33$ & $1.34$ & $1.37$ & $1.40$ & $2.69$ & $2.28$ & $2.30$ & $2.40$ \\
\texttt{cube-double-task3} & $1.00$ & $1.00$ & $1.05$ & $1.27$ & $1.00$ & $1.60$ & $2.50$ & $2.60$ \\
\texttt{cube-double-task4} & $1.32$ & $1.30$ & $1.31$ & $1.21$ & $3.38$ & $2.64$ & $2.33$ & $1.64$ \\
\texttt{cube-triple-task3} & $1.23$ & $1.24$ & $1.27$ & $1.21$ & $2.00$ & $1.77$ & $1.76$ & $1.57$ \\
\texttt{cube-triple-task4} & $1.36$ & $1.40$ & $1.45$ & $1.45$ & $2.41$ & $2.50$ & $2.63$ & $2.63$ \\
\texttt{scene-task4} & $1.00$ & $1.00$ & $1.09$ & $1.16$ & $1.20$ & $1.20$ & $1.86$ & $2.12$ \\
\texttt{scene-task5} & $1.13$ & $1.22$ & $1.36$ & $1.32$ & $3.20$ & $4.40$ & $4.33$ & $2.54$ \\
\texttt{hmaze-med-task3} & $1.41$ & $1.55$ & $1.63$ & $1.63$ & $6.17$ & $6.33$ & $6.14$ & $6.14$ \\
\texttt{amaze-giant-task4} & $1.54$ & $1.56$ & $1.52$ & $1.59$ & $4.56$ & $4.67$ & $3.91$ & $4.27$ \\
\texttt{amaze-giant-task5} & $1.02$ & $1.07$ & $1.21$ & $1.14$ & $1.20$ & $1.50$ & $1.88$ & $1.37$ \\
\midrule
\textbf{Average} & $\mathbf{1.22}$ & $\mathbf{1.26}$ & $\mathbf{1.33}$ & $\mathbf{1.35}$ & $\mathbf{2.93}$ & $\mathbf{3.15}$ & $\mathbf{3.11}$ & $\mathbf{2.77}$ \\
95\% CI & $[1.10,\,1.33]$ & $[1.14,\,1.37]$ & $[1.23,\,1.43]$ & $[1.25,\,1.45]$ & $[1.97,\,3.89]$ & $[2.08,\,4.23]$ & $[2.27,\,3.95]$ & $[1.95,\,3.59]$ \\
\bottomrule
\end{tabular}
}
\end{table*}

\newpage
\section{Inference-Time Beam Search}
\label{app:beam_search}

\looseness=-1
\Cref{tab:sweep_off_norec} provides the full per-environment breakdown of \nameshortinf{} applied to the trained \nameshorttr{} checkpoint, extending the aggregate IQM results reported in~\cref{tab:iqm_fe}. NFE $= M(1 + KB)$, where $M$ is the number of initial candidates, $K$ the number of renoising steps, and $B$ the number of completions per candidate; $K{=}0$ reduces to standard best-of-$M$. The per-environment results confirm that the gains from renoising ($K{=}1$) are consistent across task domains---manipulation, multi-object rearrangement, and locomotion---with the most notable improvements on the harder maze tasks (\texttt{hm3}: $0.59 \to 0.72$, \texttt{hm4}: $0.07 \to 0.11$).

\begin{table*}[t]
\centering
\tiny
\setlength{\tabcolsep}{3pt}
\caption{\nameshortinf{} on \nameshorttr{}. SNR$=1.5$, $\eta=0.3$. Columns grouped by $K$; sub-columns $\{B, M\}$. Success rate (mean $\pm$ std, 5 seeds, 50 eps). Best per row in $\mathbf{bold}$.}
\label{tab:sweep_off_norec}
\begin{tabular}{lcccccccc}
\toprule
\textbf{Environment} & \multicolumn{1}{c}{$K=0$} & \multicolumn{4}{c}{$K=1$} & \multicolumn{2}{c}{$K=2$} \\
\cmidrule(lr){2-2} \cmidrule(lr){3-6} \cmidrule(lr){7-8}
 & $\{1, 32\}$ & $\{1, 16\}$ & $\{2, 8\}$ & $\{4, 4\}$ & $\{4, 16\}$ & $\{1, 16\}$ & $\{4, 4\}$ \\
\midrule
\texttt{can} & $0.96 \pm 0.04$ & $0.97 \pm 0.02$ & $0.96 \pm 0.04$ & $0.97 \pm 0.03$ & $0.94 \pm 0.04$ & $0.95 \pm 0.03$ & $\mathbf{0.98 \pm 0.03}$ \\
\texttt{square} & $0.94 \pm 0.02$ & $0.94 \pm 0.03$ & $\mathbf{0.96 \pm 0.04}$ & $0.95 \pm 0.04$ & $\mathbf{0.96 \pm 0.02}$ & $0.94 \pm 0.04$ & $0.94 \pm 0.03$ \\
\texttt{cdp3} & $\mathbf{1.00 \pm 0.00}$ & $0.99 \pm 0.01$ & $\mathbf{1.00 \pm 0.00}$ & $\mathbf{1.00 \pm 0.00}$ & $\mathbf{1.00 \pm 0.00}$ & $\mathbf{1.00 \pm 0.00}$ & $\mathbf{1.00 \pm 0.00}$ \\
\texttt{cdp4} & $0.98 \pm 0.02$ & $0.99 \pm 0.03$ & $0.99 \pm 0.01$ & $\mathbf{1.00 \pm 0.00}$ & $0.98 \pm 0.02$ & $0.99 \pm 0.03$ & $0.99 \pm 0.01$ \\
\texttt{ctrp3} & $0.78 \pm 0.10$ & $0.82 \pm 0.10$ & $0.78 \pm 0.06$ & $\mathbf{0.84 \pm 0.04}$ & $0.83 \pm 0.07$ & $\mathbf{0.84 \pm 0.08}$ & $0.82 \pm 0.08$ \\
\texttt{ctrp4} & $\mathbf{0.88 \pm 0.07}$ & $0.84 \pm 0.06$ & $\mathbf{0.88 \pm 0.06}$ & $0.87 \pm 0.05$ & $0.82 \pm 0.09$ & $0.82 \pm 0.04$ & $0.84 \pm 0.06$ \\
\texttt{sc4} & $\mathbf{1.00 \pm 0.00}$ & $\mathbf{1.00 \pm 0.00}$ & $0.99 \pm 0.01$ & $0.99 \pm 0.01$ & $\mathbf{1.00 \pm 0.00}$ & $0.99 \pm 0.01$ & $\mathbf{1.00 \pm 0.00}$ \\
\texttt{sc5} & $0.98 \pm 0.02$ & $\mathbf{1.00 \pm 0.00}$ & $0.99 \pm 0.01$ & $\mathbf{1.00 \pm 0.00}$ & $0.99 \pm 0.01$ & $\mathbf{1.00 \pm 0.00}$ & $0.98 \pm 0.02$ \\
\texttt{hm3} & $0.69 \pm 0.04$ & $0.70 \pm 0.07$ & $0.63 \pm 0.04$ & $0.58 \pm 0.07$ & $\mathbf{0.72 \pm 0.11}$ & $0.68 \pm 0.10$ & $0.70 \pm 0.07$ \\
\texttt{hm4} & $0.06 \pm 0.03$ & $0.07 \pm 0.04$ & $0.10 \pm 0.04$ & $0.06 \pm 0.03$ & $\mathbf{0.11 \pm 0.04}$ & $0.10 \pm 0.05$ & $0.10 \pm 0.03$ \\
\texttt{ag4} & $0.80 \pm 0.06$ & $0.78 \pm 0.06$ & $\mathbf{0.86 \pm 0.10}$ & $0.77 \pm 0.03$ & $0.82 \pm 0.04$ & $0.75 \pm 0.05$ & $0.79 \pm 0.03$ \\
\texttt{ag5} & $\mathbf{0.92 \pm 0.04}$ & $\mathbf{0.92 \pm 0.04}$ & $0.90 \pm 0.05$ & $\mathbf{0.92 \pm 0.05}$ & $0.90 \pm 0.09$ & $0.89 \pm 0.03$ & $0.90 \pm 0.02$ \\
\midrule
IQM & $0.91\;[0.89, 0.93]$ & $0.92\;[0.90,0.93]$ & $\mathbf{0.93\;[0.91,0.95]}$ & $\mathbf{0.93\;[0.91, 0.94]}$ & $0.92\;[0.90,0.93]$ & $0.90\;[0.89,0.92]$ & $0.91\;[0.90,0.93]$ \\
\bottomrule
\end{tabular}
\end{table*}

\begin{table}[h]
\centering
\small
\setlength{\tabcolsep}{5pt}
\caption{Hyperparameters shared across all methods.}
\label{tab:hyperparams-shared}
\begin{tabular}{l l}
\toprule
\textbf{Parameter} & \textbf{Value} \\
\midrule
Optimizer & Adam \\
Learning rate & $3 \!\times\! 10^{-4}$ \\
Batch size & $256$ \\
Discount ($\gamma$) & $0.99$ \\
Target update ($\tau$) & $5 \!\times\! 10^{-3}$ \\
UTD ratio & $1$ \\
Offline / online steps & $1\text{M}$ / $1\text{M}$ \\
Replay buffer & $2\text{M}$ \\
\midrule
Policy network & MLP, $4 \!\times\! 512$, GELU \\
Critic network & MLP, $4 \!\times\! 512$, GELU, LayerNorm \\
Critic ensemble & $2$ (double Q, mean agg.) \\
Fourier embedding & $64$ dim per time axis \\
Chunking horizon ($H$) & $5$ \\
\midrule
Eval interval / episodes & $100\text{K}$ / $50$ \\
\bottomrule
\end{tabular}
\end{table}

\begin{table}[t]
\centering
\small
\setlength{\tabcolsep}{4pt}
\caption{Inference procedures. NFE = network forward evaluations per action.}
\label{tab:inference}
\begin{tabular}{l l c c c}
\toprule
\textbf{Method} & \textbf{Action selection} & \textbf{Steps} & $N$ & \textbf{NFE} \\
\midrule
QC & Best-of-$N$ (Euler) & $10$ & $32$ & $320$ \\
MVP & Best-of-$N$ (flow map) & $1$ & $32$ & $32$ \\
MVP + \nameshortinf (ours) & \nameshortinf & $K$ & $M$ & $K \!\cdot\! M$ \\
\nameshorttr (ours) & Best-of-$N$ (flow map) & $1$ & $32$ & $32$ \\
\nameshorttr + \nameshortinf (ours) & \nameshortinf & $K$ & $M$ & $K \!\cdot\! M$ \\
\bottomrule
\end{tabular}
\end{table}

\section{Implementation Details}
\label{app:implementation}

\looseness=-1
All methods share the same network architecture, critic algorithm, and training pipeline to ensure a controlled comparison. The policy is parameterized as a time-conditioned velocity field $u_\theta$: a 4-layer MLP with 512 hidden units and GELU activations. Scalar flow times are lifted to 64-dimensional sinusoidal Fourier embeddings before concatenation with the observation and noisy action. The critic follows the clipped double Q-learning framework~\citep{fujimoto2018doubleqlearning}: an ensemble of two Q-networks with the same MLP architecture (with LayerNorm) trained against a shared Bellman target using Polyak-averaged target networks ($\tau{=}0.005$). All methods use action chunking ($H{=}5$), a replay buffer of $2\text{M}$ transitions, and are trained for $1\text{M}$ offline followed by $1\text{M}$ online steps with UTD ratio 1, Adam ($\mathrm{lr}{=}3{\times}10^{-4}$), and batch size 256. Full shared hyperparameters are in~\cref{tab:hyperparams-shared}.

\looseness=-1
QC~\citep{li2025reinforcement} trains a standard CFM velocity field $v_\theta(a_t, t \mid s)$ with the straight-line interpolation objective. At inference, the ODE is integrated from $t{=}0$ to $t{=}1$ with 10 Euler steps, producing 32 candidates scored by the critic (best-of-$N$, 320 NFE total).MVP~\citep{zhan2026mean} replaces multi-step Euler integration with a single-step flow map ($K{=}1$) that directly predicts the average velocity $u_{r,t}(a_r \mid s)$ over $[r, t]$. The network takes as input $[\mathbf{s}, \mathbf{x}_r, \text{Fourier}(r), \text{Fourier}(t), \text{Fourier}(t_c), \mathbf{a}_c]$ where $(t_c, \mathbf{a}_c)$ form a conditioning axis for stochastic action generation. Training uses a progressive curriculum: diagonal-only CFM ($r{=}t$) for $5\text{K}$ steps, then the interval $[r, t]$ is annealed to the full range over $50\text{K}$ steps, and the conditioning axis is introduced after $10\text{K}$ steps with $P(t_c{=}0){=}0.5$. At inference, a single forward pass generates each candidate and best-of-$32$ selection is applied (32 NFE). \nameshort{} shares the same offline pretraining as MVP. During the online phase, it switches to the trust-region Q-gradient objective described in~\cref{sec:online_adaptation}: the offline flow map is frozen as the reference $u^{\text{off}}_{r,1}$, the Q-gradient is $\ell_2$-normalized, and the trust-region radius $\eta_{\mathrm{eff}}$ adapts per sample via Q-ensemble disagreement ($\beta{=}0.3$, cf.~\cref{eq:eta_eff}). At inference, \nameshort{} uses the same best-of-$32$ flow map selection as MVP (32 NFE). \nameshortinf{} applies the $Q$-guided beam search of~\cref{sec:inference_steering} at inference time without additional training. Starting from a trained flow map, actions are diversified via SNR-based renoising and refined over $K$ beam steps using $M$ candidates, with the trust-region projection applied at each iteration. Steering strength is controlled by $\lambda$ and actions are clipped via a straight-through estimator. Cost: $M(1 + KB)$ NFE per action. Inference procedures and computational costs are summarized in~\cref{tab:inference}. All training experiments were run on NVIDIA A100-SXM4-80GB GPUs. A full training run takes approximately 4 hours per seed. Inference-time steering evaluations were conducted on NVIDIA RTX 6000 Ada Generation GPUs (48\,GB VRAM).

\end{document}